\documentclass{article}
\usepackage[T1]{fontenc}

\usepackage[final]{colm2026_conference}

\usepackage{amsmath}
\usepackage{newtxtext,newtxmath}
\usepackage{microtype}
\usepackage{lineno}
\usepackage{amsthm}
\usepackage{mathtools}
\usepackage{hyperref}
\usepackage{url}
\usepackage{booktabs}
\usepackage{graphicx}
\usepackage{algorithm}
\usepackage{algpseudocode}
\usepackage{array}
\usepackage{multirow}
\usepackage{xcolor}
\usepackage{colortbl}
\usepackage{longtable}
\usepackage{listings}

\graphicspath{{./}}

\usepackage{tikz}
\usetikzlibrary{positioning, arrows.meta, calc, shapes.geometric, shadows}
\definecolor{structureblue}{RGB}{60,76,168}
\definecolor{innerteal}{RGB}{13,138,112}
\definecolor{scoreslate}{RGB}{71,85,105}
\definecolor{holeamber}{RGB}{201,108,18}
\definecolor{inkgray}{RGB}{52,62,78}
\newcommand{\hole}[1]{\textcolor{holeamber}{#1}}

\theoremstyle{plain}
\newtheorem{proposition}{Proposition}
\theoremstyle{definition}
\newtheorem{definition}{Definition}
\newtheorem{remark}{Remark}

\newcommand{\R}{\mathbb{R}}
\newcommand{\E}{\mathbb{E}}
\newcommand{\Inst}{\textsc{Inst}}
\newcommand{\Tune}{\textsc{Tune}}
\newcommand{\Validate}{\textsc{Validate}}
\newcommand{\Mcal}{\mathcal{M}}
\newcommand{\Tcal}{\mathcal{T}}
\newcommand{\Ccal}{\mathcal{C}}

\definecolor{corr}{RGB}{31,119,180}
\newcommand{\hc}[2]{\cellcolor{corr!#1}#2}       
\newcommand{\sci}[2]{$#1\!\times\!10^{#2}$}
\newcommand{\bsci}[2]{$\mathbf{#1\!\times\!10^{#2}}$}
\newcommand{\dnar}{$\downarrow$}
\newcommand{\upar}{$\uparrow$}
\newcommand{\dt}{\,$\cdot$\,}
\newcommand{\nad}{\textendash}
\newcommand{\blkrow}[1]{\multicolumn{7}{@{}p{\linewidth}@{}}{\rule{0pt}{2.4ex}#1}}
\newcommand{\famrow}[1]{\multicolumn{7}{@{}l}{\rule{0pt}{2.7ex}\sffamily\bfseries #1}}

\title{A Hybrid Nested Harness for Decoupling\\Structure and Parameters in LLM-Driven Optimization
}

\author{Víctor Gallego \\
Komorebi AI Technologies}

\begin{document}

\ifcolmsubmission
\linenumbers
\fi

\maketitle

\begin{abstract}
In evolutionary algorithms powered by language models, the LLM acts as a single operator that simultaneously
updates structural components (like control flow) and continuous parameters. While LLMs can be good at the first,
they are not efficient at the second, wasting tokens taking discrete jumps inside a trial and error loop.
We resolve this by formalizing a hybrid nested search, in which an outer loop has the LLM propose a structural sketch, with numeric gaps, and
an inner numerical optimizer tunes the sketch. Both the outer and inner solvers are pluggable: any text-based optimizer can be combined with a zero-order optimizer (CMA-ES), gradient-based routines, or MCMC samplers. We validate our framework across three scientific domains:
(i) meta-optimizers on closed-form test functions, (ii) code-based policies for systems research and social dilemmas; and (iii) approximate Bayesian inference tasks. Across all three, the hybrid optimizer is superior to both vanilla LLM-driven search and pure numerical optimization baselines. \\
Code at: \href{https://github.com/vicgalle/hybrid-nested-search}{github.com/vicgalle/hybrid-nested-search}
\end{abstract}

\section{Introduction}
\label{sec:intro}

We adopt the setting in which a frozen LLM $\Mcal$ acts
as a synthesis operator inside a loop. 
At each iteration the model generates a code-based artifact (a policy function, a probabilistic-model reparameterization, a training script, a GPU kernel, etc.), a harness evaluates it under a fitness (scoring) function $f$, and structured feedback is prompted to the next generation.
This is the architecture behind recent LLM-based discovery systems in different scientific domains~\citep{romeraparedes2024, novikov2025}. Vanilla autoresearch lets $\Mcal$ propose the entire artifact at once in a single optimization turn.

We argue that an artifact (the optimization solution) can be decomposed into two qualitatively different decision types. The
\textbf{structure} (sketch) is the program topology: control structures like branches, hard logic, auxiliary function definitions. This is text-based and benefits from the model's knowledge prior. The \textbf{parameter vector} is the numeric
values embedded in that structure: thresholds, learning rates, etc. This is a continuous (possibly
mixed-integer) black-box optimization problem. Language models are poor
optimizers~\citep{yang2023}, wasting huge amounts of compute resources whereas \emph{traditional} numerical solvers are dramatically more efficient.
Our idea is thus to formalize this dual structure as a nested search loop, and to propose a practical algorithm that can leverage any text-based optimizer, any numerical optimizer, and combine them by \emph{sketching} (see Figure \ref{fig:framework}).

Our contributions are these: (i) we formalize hybrid nested search as a bilevel
objective and give an algorithm in which a frozen LLM searches structure and an inner solver tunes parameters (Sec.~\ref{sec:framework}); (ii) we identify \emph{parametric aliasing} as the failure mode of joint search and prove that the
inner loop de-aliases the outer ranking
(Sec.~\ref{sec:dealias}); (iii) we validate this factorization on three problem
families: meta-optimizers
(Sec.~\ref{sec:meta}), code-based policies for systems research and social dilemmas
(Sec.~\ref{sec:policies}), and approximate Bayesian inference tasks (Sec.~\ref{sec:bayes}).
The result holds across inner solvers: CMA-ES \citep{hansen2001}, VI \citep{kucukelbir2017},  NUTS \citep{hoffman2014}, outer structural
optimizers (a simple LLM loop and reflective prompt evolution \citep{agrawal2025}), and different language model families.

\section{Hybrid Nested Search}
\label{sec:framework}

\begin{figure}[t]
\centering
\resizebox{\textwidth}{!}{%
\begin{tikzpicture}[
  font=\sffamily,
  >={Stealth[round]},
  card/.style={rounded corners=6pt, draw=#1, line width=1pt, fill=#1!6,
    minimum width=3.7cm, minimum height=2.9cm,
    drop shadow={opacity=0.18, shadow xshift=1.3pt, shadow yshift=-1.3pt}},
  hdr/.style={font=\bfseries\small, text=white},
  sub/.style={font=\itshape\scriptsize},
  badge/.style={rounded corners=3pt, draw=innerteal!55, line width=0.6pt,
    fill=innerteal!10, text=innerteal!72!black, font=\scriptsize\bfseries,
    inner xsep=4pt, inner ysep=2.2pt},
  arr/.style={line width=1.2pt},
  lbl/.style={font=\scriptsize, fill=white, inner sep=1.6pt, text=inkgray},
]

\node (A) [card=structureblue] at (1.85,1.45) {};
\begin{scope}
  \clip[rounded corners=6pt] (A.south west) rectangle (A.north east);
  \fill[structureblue] (A.north west) rectangle ([yshift=-0.62cm]A.north east);
\end{scope}
\node[hdr] at ([yshift=-0.31cm]A.north) {Frozen LLM $\Mcal$};
\begin{scope}[shift={([xshift=0.21cm,yshift=-0.31cm]A.north west)}, white, line width=0.5pt]
  \foreach \a in {0,60,120}{\draw (\a:0.09) -- (\a+180:0.09);}
\end{scope}
\node[sub, text=structureblue!82!black] at (1.85,1.99) {propose structure $\tau$ + manifest};
\node[align=left, text=inkgray, font=\ttfamily\scriptsize, inner sep=0pt] at (1.85,1.12)
  {def policy(env, p):\\ ~~if w(env) > \hole{[$\theta_1$]}:\\ ~~~~clean(\hole{[$\theta_2$]}*N)};
\node[font=\tiny, text=holeamber!88!black] at (1.85,0.34)
  {manifest: $(\ell_j,\,u_j,\,\mathrm{type}_j,\,\bar\theta_j)$};

\node (B) [card=innerteal] at (6.75,1.45) {};
\begin{scope}
  \clip[rounded corners=6pt] (B.south west) rectangle (B.north east);
  \fill[innerteal] (B.north west) rectangle ([yshift=-0.62cm]B.north east);
\end{scope}
\node[hdr] at ([yshift=-0.31cm]B.north) {Inner optimizer \textsc{Tune}};
\node[sub, text=innerteal!72!black] at (6.75,1.99) {fill the holes: $\bar\theta\!\to\!\theta^\star(\tau)$};
\coordinate (Lc) at (6.75,1.00);
\foreach \rx/\ry/\op in {1.28/0.66/12, 0.97/0.50/20, 0.66/0.34/31, 0.36/0.19/44}
  \fill[innerteal!\op, draw=innerteal!30, line width=0.2pt]
    (Lc) ellipse [x radius=\rx cm, y radius=\ry cm];
\coordinate (g) at ($(Lc)+(-0.97,0.40)$);
\node[star, star points=5, star point ratio=2.3, fill=holeamber,
  draw=holeamber!65!black, line width=0.3pt, inner sep=0pt, minimum size=9pt] (st) at (Lc) {};
\draw[innerteal!82!black, line width=0.9pt, densely dashed, ->]
  (g) to[out=-22,in=145] ($(st)+(-0.05,0.05)$);
\fill[white] (g) circle (2.3pt);
\draw[innerteal!82!black, line width=0.8pt] (g) circle (2.3pt);
\node[font=\tiny, text=inkgray] at ($(g)+(-0.19,0.02)$) {$\bar\theta$};
\node[font=\tiny, text=holeamber!82!black] at ($(st)+(0.02,-0.27)$) {$\theta^\star$};

\node (C) [card=scoreslate] at (11.65,1.45) {};
\begin{scope}
  \clip[rounded corners=6pt] (C.south west) rectangle (C.north east);
  \fill[scoreslate] (C.north west) rectangle ([yshift=-0.62cm]C.north east);
\end{scope}
\node[hdr] at ([yshift=-0.31cm]C.north) {Score \& promote};
\node[sub, text=scoreslate!80!black] at (11.65,1.99) {eval on the \textbf{tuned} value};
\node[align=center, font=\small, text=inkgray] at (11.65,1.20)
  {$\widehat F(\tau)\approx\displaystyle\max_{\theta} f(\tau,\theta)$};
\node[align=center, font=\scriptsize, text=inkgray] at (11.65,0.43)
  {$(1\!+\!1)$: keep $\tau$ if better\\than current incumbent};

\draw[arr, structureblue!85!black, ->] (A.east) -- (B.west)
  node[lbl, midway, above=1pt]{$\tau$, manifest};
\draw[arr, innerteal!85!black, ->] (B.east) -- (C.west)
  node[lbl, midway, above=1pt]{$\theta^\star,\ \widehat F(\tau)$};

\draw[arr, scoreslate, ->] (C.north) .. controls +(0,1.55) and +(0,1.55) .. (A.north);
\node[draw=scoreslate!40, fill=scoreslate!8, rounded corners=3pt, inner sep=3pt,
  font=\scriptsize, text=scoreslate!92!black] at (6.75,4.02)
  {outer loop \,\textbullet\, feedback: $\widehat F(\tau_{k-1}),\ \theta^\star,$ diagnostics};

\draw[innerteal!55, line width=0.7pt, densely dotted] (B.south) -- (6.75,-0.30);
\node[font=\tiny\itshape, text=innerteal!72!black] at (6.75,-0.18) {pluggable inner solver};
\node[badge] (sb1) at (5.78,-0.62) {CMA-ES};
\node[badge] (sb2) at (6.83,-0.62) {VI};
\node[badge] (sb3) at (7.74,-0.62) {NUTS};

\node[rounded corners=5pt, draw=holeamber!45, line width=0.8pt, fill=holeamber!7,
  text=inkgray, font=\scriptsize, align=center, text width=12.7cm, inner sep=5pt]
  at (6.75,-1.52)
  {\textbf{De-aliasing.} Vanilla joint search ranks structures by the
   untuned guess $f(\tau,\bar\theta)$; hybrid ranks them by the parametric optimum
   $\widehat F(\tau)$. The gap $\Delta(\tau)=F(\tau)-f(\tau,\bar\theta)$ correlates with the
   advantage of hybrid vs. vanilla joint search.};

\end{tikzpicture}%
}
\caption{\textbf{Hybrid nested search.} A frozen LLM (left) proposes a text sketch $\tau$ with numeric gaps and a manifest of bounds, types, and a
suggested value $\bar\theta$. An inner numerical optimizer (center) tunes the guess $\bar\theta$ to the optimum
$\theta^\star(\tau)$, and returns $\widehat F(\tau)\approx\max_\theta f(\tau,\theta)$. The artifact is scored and
promoted on that tuned value (right), and feedback is provided for the outer loop's next iteration. Our framework is agnostic to the choice of inner and outer solvers. Evaluating on the tuned
value in place of the LLM guess removes the parametric aliasing of joint search.}
\label{fig:framework}
\end{figure}
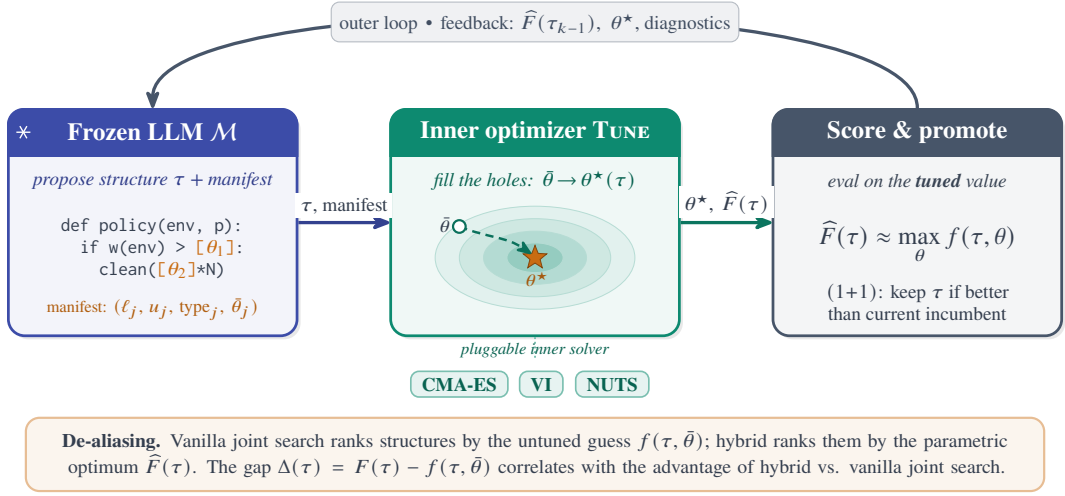

Our framework is displayed in Figure~\ref{fig:framework}: a frozen LLM proposes a structural
sketch with numeric holes, an inner optimizer fills the holes to their parametric
optimum, and the structure is scored on that tuned result before the next
proposal. Note this nested factorization is agnostic to the choice of inner and outer solvers, and we leverage text-based outer optimizers to learn from structured feedback (not only the score, but also compiler diagnostics).

\subsection{Decomposition and the parameter manifest}

\begin{definition}[Sketch decomposition]
A candidate program is a pair $\langle\tau,\theta\rangle$ where $\tau\in\Tcal$ is
a structure drawn from a discrete sketch space (expressed via text),
exposing $h(\tau)\in\mathbb{N}$ numeric \emph{holes}, and
$\theta\in\Theta(\tau)\subseteq\R^{h(\tau)}$ is a parameter vector. An
instantiation map $\Inst:(\tau,\theta)\mapsto$ executable code fills the holes with actual values.
The parameter space $\Theta(\tau)$ depends on $\tau$, so while the outer loop searches over structures, the number of parameters can vary.
\end{definition}

Rather than parsing numeric literals out of generated code,
we require $\Mcal$ to output, for each gap $j$, a \emph{manifest}: a tuple
$m(\tau)=\{(j,\ \ell_j,\ u_j,\ \mathrm{type}_j,\ \bar\theta_j)\}_{j=1}^{h(\tau)}$,
giving a lower bound $\ell_j$, upper bound $u_j$, a type
$\mathrm{type}_j\in\{\mathrm{cont},\mathrm{int},\mathrm{log},\mathrm{pow2}\}$, and
a suggested value $\bar\theta_j$, the LLM's own best guess. The inner optimizer takes the vector $\bar\theta$ as a warm start, at each outer iteration. The manifest is a lightweight interface that allows the LLM to delegate the numeric search to a specialized solver, 
making this choice explicit compared to other numeric values that do not need to be tuned (e.g., dimension sizes).

\subsection{Bilevel objective and algorithm}

Let $f(\tau,\theta)$ be the fitness function of the problem (welfare for a policy, sampling
efficiency for an inference task, cost savings for a systems algorithm, ...). Define the \emph{structural value} $F(\tau)$ as the
optimized fitness for a fixed structure $\tau$,
\begin{equation}
F(\tau) \;=\; \max_{\theta\in\Theta(\tau)} f(\tau,\theta),
\qquad
\theta^\star(\tau) \;=\; \arg\max_{\theta\in\Theta(\tau)} f(\tau,\theta).
\label{eq:F}
\end{equation}
Hence, we are interested in searching for the structure that maximizes its structural value:
\begin{equation}
\tau^\star \;=\; \arg\max_{\tau\in\Tcal} F(\tau)
\;=\; \arg\max_{\tau\in\Tcal}\ \max_{\theta\in\Theta(\tau)} f(\tau,\theta).
\label{eq:bilevel}
\end{equation}
This is a bilevel optimization problem, over two different scales. A numerical optimizer solves the inner
level \eqref{eq:F}; the LLM searches the outer level \eqref{eq:bilevel}. The inner
solver returns an estimate of \eqref{eq:F} within a budget
$B_{\mathrm{in}}$ of fitness evaluations, $\widehat F(\tau)=\Tune\big(f(\tau,\cdot),
\Theta(\tau),\bar\theta(\tau),B_{\mathrm{in}}\big)$, using
$\bar\theta(\tau)$ as the warm start.

\begin{algorithm}[t]
\caption{Hybrid nested structural / parametric search ((1+1) outer loop)}
\label{alg:hybrid}
\begin{algorithmic}[1]
\Require task $T$, LLM $\Mcal$, system prompt $p$, outer budget $K$, inner budget $B_{\mathrm{in}}$
\Ensure incumbent structure $\tau^\star$, parameters $\theta^\star$, value $V^\star$
  \State $(\tau_0,m_0)\leftarrow \Mcal(p,\text{``propose initial structure + manifest''})$
  \State $\widehat F_0,\theta^\star_0 \leftarrow \Tune(f(\tau_0,\cdot),\Theta(\tau_0),\bar\theta(\tau_0),B_{\mathrm{in}})$
  \State $\tau^\star,\theta^\star,V^\star \leftarrow \tau_0,\theta^\star_0,\widehat F_0$
  \For{$k=1,\dots,K$}
    \State $(\tau_k,m_k)\leftarrow \Mcal(p,\ q(\tau^\star,V^\star,\theta^\star,\mathrm{diag}))$ \Comment{LLM mutates structure}
    \If{$\neg\Validate(\tau_k)$} \textbf{continue} \Comment{check sketch is correct}
    \EndIf
    \State $\widehat F_k,\theta^\star_k \leftarrow \Tune(f(\tau_k,\cdot),\Theta(\tau_k),\bar\theta(\tau_k),B_{\mathrm{in}})$ \Comment{inner tune}
    \If{$\widehat F_k > V^\star$}
      \State $\tau^\star,\theta^\star,V^\star \leftarrow \tau_k,\theta^\star_k,\widehat F_k$ \Comment{promote on tuned value}
    \EndIf
  \EndFor
  \State \Return $\tau^\star,\theta^\star,V^\star$
\end{algorithmic}
\end{algorithm}

\paragraph{Inner loop: a pluggable parametric optimizer.} \Tune{} is a
routine agnostic to the LLM and chosen to refine the inner fitness. For a non differentiable black-box
objective, we use CMA-ES with 
$\bar\theta(\tau)$ as the initial guess, projecting integer and power-of-2 coordinates onto their
feasible lattice. For differentiable magnitudes, the natural choice is a gradient-based
routine; and these can be replaced with MCMC methods in case we are interested in optimizing a distribution rather than a single point.
The factorization holds whichever solver fills the holes in the sketch; only the solver
changes. The fitness function should be a robust score (e.g. using the mean across several samples), so
the inner loop does not overfit.

\paragraph{Outer loop: LLM structural proposal.} At outer step $k$ the LLM
proposes a new structure $\tau_k$ conditioned on the previous structure $\tau_{k-1}$, its optimized
value $\widehat F(\tau_{k-1})$, the optimized parameters, and other domain diagnostics, provided as text feedback.
With this feedback, we can expose the LLM to the $\widehat F(\tau_{k-1})$ estimate, the
structure's potential at its parametric best, instead of just the score of the model's
original guess. That is, the LLM is informed with the ground truth signal from the inner optimizer, so it is better informed for the next iteration. The accept rule may be the usual (1+1) rule (keep only the best incumbent, Algorithm~\ref{alg:hybrid}) or
population-based; our results also hold when a more sophisticated method, such as GEPA \citep{agrawal2025}, optimizes the outer loop (Sec.~\ref{sec:policies}).

\subsection{Parametric de-aliasing}
\label{sec:dealias}
Recent work has shown that, within text-based structural optimizers, providing the LLM operator with multiple signals can break the aliasing of different failure modes that collapse to the same scalar reward~\citep{cheng2024trace, agrawal2025, gallego2026}. We generalize this aliasing notion to the hybrid setting, in which the aliased quantities are the structural quality and the quality of the parameter guesses. 
Vanilla, pure text-based search observes the LLM's untuned score
$g_{\mathrm{van}}(\tau)=f(\tau,\bar\theta(\tau))$. Hybrid search observes instead
$g_{\mathrm{hyb}}(\tau)=\widehat F(\tau)\approx F(\tau)$. Let's define the
\emph{parameter-tuning gap} as
\begin{equation}
\Delta(\tau)\;=\;F(\tau)-f(\tau,\bar\theta(\tau))\;\ge\;0,
\label{eq:gap}
\end{equation}
that is, the fitness a structure leaves on the table under the LLM's guessed parameters.

\begin{proposition}[Parametric aliasing in LLM-driven search]
\label{prop:alias}
Assume the outer operator's accept decision is monotone in the observed score (a
structure is preferred to the incumbent iff its observed score is higher), and
that the inner solver is $\varepsilon$-accurate, $|\widehat F(\tau)-F(\tau)|\le
\varepsilon$. Then: (1) vanilla and hybrid induce the same preference between
$\tau,\tau'$ iff $\Delta(\tau)-\Delta(\tau')$ does not flip the sign of
$F(\tau)-F(\tau')$; (2) vanilla can reject the structural optimum $\tau^\star$ in
favor of an inferior $\tau'$ whenever
$\Delta(\tau^\star)-\Delta(\tau')>F(\tau^\star)-F(\tau')$, while hybrid never
commits this error up to a margin $2\varepsilon$; (3) under an i.i.d.\ proposal
model with $F\perp\Delta$ (i.e., they are independent), the expected hybrid advantage decomposes into a
\emph{selection} term, positive on the inversions of (1) and increasing in the
cross-structure variance of $\Delta$, and a \emph{tuning} term, present on every
accepted structure and increasing in the mean $\E[\Delta]$; it vanishes iff
$\Delta\equiv0$.
\end{proposition}

\noindent We place the proof in Appendix~\ref{app:proof}. When the LLM's guesses for parameters are already near optimal ($\Delta\approx0$, e.g.\ a
policy that doesn't require numerical parameters) the two search methods coincide. When good structures carry
constants the model cannot guess (learning rates,
preconditioner scales, etc.), vanilla search discards them prematurely, while the inner loop from hybrid search
can de-alias the signal. We arrive at a falsifiable prediction:
\begin{equation}
(\text{hybrid advantage})\ \propto\ \E_{\tau\sim\Mcal}\big[\Delta(\tau)\big].
\label{eq:eq4}
\end{equation}
Here the \emph{hybrid advantage} is $\max_{\tau}\widehat F(\tau)-\max_{\tau}g_{\mathrm{van}}(\tau)$,
the best delivered value of each search method (hybrid vs vanilla). Throughout our experiments we will measure $\Delta(\tau)$ empirically, and see that the hybrid advantage grows with it.

\subsection{On the cost of hybrid search}
\label{sec:budget}

Let $c_{\mathrm{llm}}$ be the cost of one LLM proposal (to generate a structure $\tau$), and let $c_{\mathrm{ev}}$ be the cost of one evaluation of the fitness function $f(\tau,\theta)$. Over $K$ outer optimization steps, the total cost of hybrid search is $C_{\mathrm{hyb}}=K(c_{\mathrm{llm}}+B_{\mathrm{in}}c_{\mathrm{ev}})$, against $C_{\mathrm{van}}=K c_{\mathrm{llm}}$ as the cost for vanilla, LLM-driven search. That is, the hybrid search costs a factor $(1+c_{\mathrm{ev}}/c_{\mathrm{llm}} B_{\mathrm{in}})$ more. In practice, a rollout is on the order of seconds, while a high thinking budget LLM call is on the order of minutes, so $c_{\mathrm{ev}}/c_{\mathrm{llm}}\sim10^{-2}$--$10^{-3}$. With an inner budget $B_{\mathrm{in}}\sim100$, this adds only $10$--$100\%$ overhead while exhausting each expensive LLM proposal. For the experiments here, the inner cost is negligible compared to the LLM.

\begin{remark}[Expensive evaluation domains]
\label{rem:expensive}
The cost comparison inverts when $c_{\mathrm{ev}}$ is large, e.g.\ when each evaluation is
a full training run. There an optimizer like CMA-ES (that typically requires tens to hundreds of evals) is the wrong inner
solver; one would substitute for a sample efficient optimizer or amortize the tuning across
structures. Our factorization still holds, but the inner solver changes.
The Bayesian inference instantiation of Sec.~\ref{sec:bayes} is this case:
the inner loop is gradient-based VI or the sampler's own warmup.
\end{remark}

A stronger inner optimizer of the fitness score is also a stronger
overfitter of it. We therefore (i) make the fitness function an aggregate, so it cannot trade robustness for a single point win,
and (ii) keep a held-out evaluation function, evaluated once on the tuned
incumbent at a sample never folded into feedback. We exercise this design choice 
with simulation-based calibration in Sec.~\ref{sec:bayes}.

\section{Experiments}
\label{sec:exp}

We evaluate our harness over three problem families. Each is a
lightweight sketch interface on top of an existing evolutionary harness, so it is straightforward to implement hybrid search from a vanilla LLM-based process. Throughout the domains we compare three approaches: \textbf{vanilla autoresearch} (LLM
joint search, no inner loop), \textbf{pure
numerical} (the inner solver over a fixed structure, no structural
search), and our \textbf{hybrid} (Algorithm~\ref{alg:hybrid}). Unless noted, the
inner budget is $B_{\mathrm{in}}=100$ and the objective function is withheld from the LLM:
its prompts show only the artifact interface and a structural description, so the $\Delta$ gap is the LLM's genuine miscalculation measured empirically. Table~\ref{tab:main} summarizes one representative task across all three families, the subsections below expand the analysis, and Appendix~\ref{app:master} shows the entire list of results.

\begin{table}[t]
\centering
\footnotesize
\setlength{\tabcolsep}{4pt}
\renewcommand{\arraystretch}{1.15}
\caption{\textbf{Representative results} across the three problem families
(full results list in
Appendix~\ref{app:master}). Three search strategies: \textbf{vanilla} (LLM joint search),
\textbf{numerical}-only (numerical solver on a fixed structure: an \emph{oracle}
CMA-ES for meta-optimizers, or a hand-tuned baseline elsewhere), and our
\textbf{hybrid}. In \textbf{bold} we mark the better of vanilla/hybrid; ``\nad'' an strategy not run for that case. $\E[\Delta]$ is the
empirical tuning gap. \colorbox{corr!45}{\strut Shading} of $\E[\Delta]$ and adv.\
columns is a heatmap, darker $=$ larger: the two columns sharing a
gradient is a visual check that advantage\,$\propto\E[\Delta]$
(Eq.~\ref{eq:eq4}). Advantage units: loss difference van$-$hyb (meta-optimizers); \% cost or
welfare achieved (policies); nats or decades ESS/grad (Bayes).}
\label{tab:main}
\begin{tabular}{@{}llccccc@{}}
\toprule
Task & Model & Vanilla & Num. & Hybrid & $\E[\Delta]$ & Hyb.\ adv. \\
\midrule
\blkrow{\textbf{Meta-optimizers} --- (1+1) LLM, CMA-ES inner; loss \dnar}\\
\cmidrule(l{0.2em}r{0.2em}){1-7}
rosenbrock & GLM-5.2 & \sci{2.1}{-4}   & \sci{4.2}{-13} & \bsci{1.5}{-12} & \hc{10}{$0.28$} & \hc{10}{\sci{+2.1}{-4}} \\
ellipsoid  & GLM-5.2 & \bsci{2.3}{-40} & \sci{8.6}{-14} & \sci{3.0}{-28}  & \hc{23}{$1.1$}  & \hc{10}{tie} \\
rastrigin  & GLM-5.2 & $19.2$          & $9.95$         & $\mathbf{1.22}$ & \hc{45}{$11.0$} & \hc{51}{$+18.0$} \\
ackley     & GLM-5.2 & $13.1$          & \sci{4.4}{-12} & \bsci{3.3}{-9}  & \hc{45}{$11.4$} & \hc{49}{$+13.1$} \\
schwefel   & GLM-5.2 & $405$           & $358$          & $\mathbf{297}$  & \hc{58}{$43.0$} & \hc{58}{$+108$} \\
\addlinespace[2pt]
\blkrow{\textbf{Executable policies} --- (1+1) LLM, CMA-ES inner; hardest regime per task}\\
\cmidrule(l{0.2em}r{0.2em}){1-7}
Can't Be Late\,\dt costly (\$\,\dnar) & Gemini 3.5 Flash & $126.9$ & $124.4$ & $\mathbf{120.5}$ & $5.5$  & $5.0\%$ \\
Cloudcast\,\dt inter (\$\,\dnar)      & Gemini 3.5 Flash & $213.3$ & $317.4$ & $\mathbf{168.7}$ & $11.7$ & $20.9\%$\,/\,$1.9\times$ \\
Cleanup\,\dt heavy ($U$\,\upar)       & Gemini 3.5 Flash & $0.34$  & $0.26$  & $\mathbf{0.58}$  & $0.37$ & $1.70\times$ \\
\addlinespace[2pt]
\blkrow{\textbf{Approximate Bayesian inference} --- (1+1) LLM, VI or NUTS inner}\\
\cmidrule(l{0.2em}r{0.2em}){1-7}
gauss\_rot \,(VI)       & Gemini 3.5 Flash & $-33.4$ & \nad          & $\mathbf{4.30}$  & \nad & $+37.7$ nats \\
funnel \,(NUTS)         & Gemini 3.5 Flash & \nad    & \sci{8.4}{-5} & $\mathbf{0.111}$ & \nad & $+3.12$ dec \\
eight\_schools \,(NUTS) & Gemini 3.5 Flash & \nad    & \sci{3.8}{-4} & $\mathbf{0.056}$ & \nad & $+2.17$ dec \\
horseshoe \,($D{=}18$, NUTS)         & Gemini 3.1 Pro   & \nad    & \sci{6.2}{-4} & \bsci{1.48}{-2}  & \nad & $+1.37$ dec \\
banana \,($D{=}10$, VI)              & Gemini 3.5 Flash & $-74.1$ & $-174$        & $\mathbf{-13.2}$ & \nad & $+60.9$ nats \\
\bottomrule
\end{tabular}
\end{table}

\subsection{Meta-optimizers}
\label{sec:meta}

The goal is to write the
optimization algorithm (the structure: momentum, adaptive gradients, restarts, etc)
for a hidden 2-D objective, while CMA-ES tunes that algorithm's hyperparameters. The LLM only sees a black-box gradient oracle and
a starting point, and never the function nor its name. We use five
objectives spanning different hard optimization regimes
(\texttt{rosenbrock}, an ill-conditioned \texttt{ellipsoid},
multimodal \texttt{rastrigin}, \texttt{ackley}, and the deceptive
\texttt{schwefel}, whose gradient points away from the global optimum); we test three different LLMs, and $N=3$ proposals per strategy.

\begin{figure}[t]
\centering
\includegraphics[width=\textwidth]{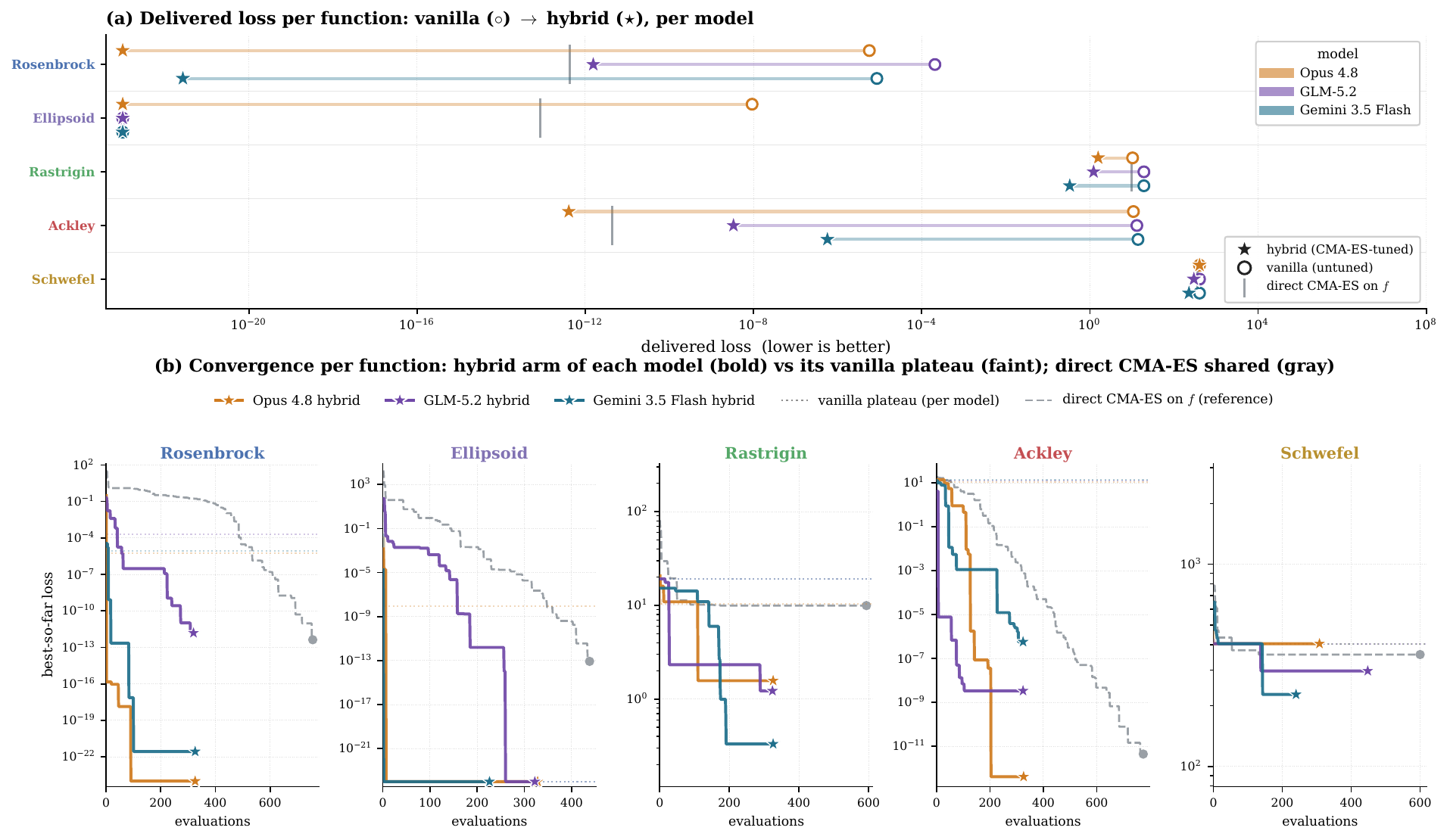}
\caption{\textbf{Meta-optimizers (Sec.~\ref{sec:meta}).} (a) Delivered loss per
function, vanilla ($\circ$) $\to$ hybrid ($\star$), for three proposer models
(lower is better, log scale). Hybrid moves the delivered loss by up to twelve
orders of magnitude on \texttt{ackley} and \texttt{rosenbrock}. (b) Best-so-far
convergence; the LLM-proposed-and-tuned structure (bold) beats CMA-ES applied
directly to the objective (gray) on the multimodal \texttt{rastrigin} /
\texttt{ackley}, where plain CMA-ES has no globally-aware structure.}
\label{fig:meta}
\end{figure}

\paragraph{Avoiding the aliasing trap.} A momentum or adaptive optimizer at
default parameters overshoots and diverges: the LLM's untuned guess scores
$\infty$ or plateaus, so pure text-based search concludes the structure is bad and discards
it. Hybrid hands the proposed structure to CMA-ES, which finds a stable
configuration: across \texttt{ackley}, \texttt{rastrigin}, and \texttt{rosenbrock} hybrid search
improves vanilla by up to many orders of magnitude and even beats the oracle
CMA-ES on the multimodal cases (Table~\ref{tab:main}; full result list in
Table~\ref{tab:master}, block~A). The two search strategies tie on \texttt{ellipsoid} (the adaptive proposal already sits near the
basin at its default), \texttt{schwefel} is the hardest, and while no method achieves the global optimum, hybrid search is the
best performer again. See Figure~\ref{fig:meta} for the convergence curves. As can be seen, hybrid search tipycally achieves lower loss than the oracle CMA-ES with much less function evaluations.

\subsection{Executable policies for systems and social dilemmas}
\label{sec:policies}

We move to code-based artifacts evaluated by simulators, where the objective is no
longer closed-form. We use three benchmark tasks with the same three strategies, each
minimizing economic cost or maximizing utility.

\paragraph{Cloud infrastructure algorithms.} Two algorithms from a cloud-systems
research benchmark. \emph{Can't Be Late}~\citep{wu2024} is a spot/on-demand
scheduler: the LLM proposes, at every step, how the state (remaining work, slack to
deadline, restart overhead) maps to an action, and CMA-ES tunes the slack buffers
and pressure thresholds it exposes; cost is dollars under a hard deadline.
\emph{Cloudcast}~\citep{wooders2024} is a multi-cloud broadcast router: the LLM proposes the routing
topology (shared relay trees, hub routing, $k$-shortest multipath) and CMA-ES
tunes parameters like weights tolerances. To save space we report the
highest difficulty regime of each (costly restarts and inter-cloud routing); hybrid search
wins in every regime evaluated (Figure~\ref{fig:cloud-cleanup}a,b;
Appendix~\ref{app:cloud}). Our harness delivers the lowest cost for both proposer
LLMs. On inter-cloud Cloudcast, hybrid (\$169, Gemini) beats joint LLM-driven search (\$213)
and pure CMA-ES (\$317), a $1.9\times$ improvement over pure CMA-ES that tunes a hand designed topology.

\paragraph{Sequential social dilemma (the Cleanup game).} A multi-agent team policy for
the Cleanup public goods gridworld~\citep{hughes2018}, an instance of a sequential
social dilemma~\citep{leibo2017}: there is a tradeoff between collecting apples (individual rewards) and contributing to cleaning the river
(collective welfare). The LLM generates a policy at each round, that maps  environment variables (like pollution level)
to the role of each agent. These policies expose parameters
that CMA-ES tunes; fitness is utilitarian welfare $U$, evaluated by self-play. We
run three seeds of increasing difficulty and
aggregate over three proposer models. Hybrid search achieves the highest welfare in
9/9 model$\times$seed runs (Figure~\ref{fig:cleanup},
Appendix~\ref{app:cleanup}), and its advantage over pure CMA-ES grows
monotonically with difficulty: structural
discovery matters more as the hidden dynamics get harsher. Two findings are of interest. First, a common ceiling: all three models independently
converge on the same algorithmic family (deadzone $+$ ramp $+$ saturation cap), optimized
to the same welfare (heavy $=0.582$ for all three). Hybrid search
lifts policies onto a ceiling none reaches without tuning. Second, a clear
Proposition~\ref{prop:alias} application: on heavy (Gemini), a piecewise ramp policy is
the worst untuned proposal ($U=-0.10$) yet the best tuned one
($U=0.582$); vanilla search discards this, promotes a sigmoid policy, and peaks at
$U=0.480$, leaving $\sim$21\% to aliasing.

\paragraph{Orthogonality to the outer optimizer.} Our hybrid factorization is agnostic to which structural optimizer is adopted. 
We re-ran all three tasks with loop of
Algorithm~\ref{alg:hybrid} replaced by a reflective, Pareto-frontier
prompt evolution method (GEPA, \cite{agrawal2025}) as the outer operator, keeping CMA-ES as the inner
\Tune. The qualitative result prevails: hybrid search beats
tuner-free GEPA in every regime with a non-negligible $\Delta$ (6/8 regimes, the two lowest gap regimes tie, and de-aliasing attributes those ties to structure selection rather than to the absence of a tuner). A stronger text-based optimizer recovers more of the gain on its own, so
the marginal value of the inner loop shrinks from $\sim$10--20\% (1+1 evolutionary loop)
to $\sim$5--10\% (GEPA). Yet the hybrid parameterization advantage never reverses where tuning matters
(Appendix~\ref{app:gepa}).

\subsection{Approximate Bayesian inference}
\label{sec:bayes}

This instantiation is the most distant test of whether the factorization
concerns \emph{variable types} or a specific solver. The
artifact is a \emph{reparameterization}: a diffeomorphism
$\theta=\mathrm{reparam}(z;\,\text{holes})$, $z\sim\R^D$, under which the host runs
inference in $z$-space with the change-of-variables Jacobian taken by autodiff
(so the LLM writes \emph{arbitrary} differentiable code instead of picking from a
menu). The structure $\tau$ is which latents to transform and the bijector family
(non-centering, affine, nonlinear warps); the holes are the transform constants
(scales, interpolation fractions). The inner \Tune{} is no longer
CMA-ES: it is variational inference (Adam on the ELBO of a fixed standard-normal
guide, where the transform carries all location and scale) or NUTS
itself, whose warmup mass-matrix adaptation serves as the inner tuner (untuned
$=$ unit metric, tuned $=$ adapted metric). The held-out gate is a longer
multi-chain run scoring divergences, $\hat R$, and effective sample size.

\paragraph{An MCMC-specific reading of Eq.~\eqref{eq:eq4}.} NUTS warmup already
adapts a \emph{linear} preconditioner (the mass matrix), so affine
reparameterizations are redundant with it ($\Delta\approx0$, a tie under
NUTS-inner) while nonlinear ones (a funnel's non-centering) are not fixable by any
mass matrix (genuine $\Delta$, hybrid wins). This predicts a VI-vs-NUTS
\emph{asymmetry on the affine controls}: an ill-conditioned Gaussian ties under
NUTS yet is a large VI gain, since a fixed standard guide cannot rescale without
the transform. Table~\ref{tab:master} (blocks~F--G) confirms both halves on a
four-model suite---under VI the
affine Gaussians win ($+37.7/+9.3$ nats) yet tie under NUTS, while under NUTS the
nonlinear funnel and eight-schools win ($+3.1/+2.2$ decades) yet the affine
controls tie.

\paragraph{Reference-free correctness, and surplus on a hard model.} A
practitioner has no gold posterior, so we replace the moment oracle with a
\emph{reference-free} gate: simulation-based calibration (SBC)~\citep{talts2018}. For a fixed
reparam, repeatedly draw $\theta^\star\sim$prior, $y\sim p(\cdot\mid\theta^\star)$,
sample the posterior in $z$-space, and rank $\theta^\star$ among the draws; under
a correct sampler-plus-transform the ranks are uniform, and deviation (a
Bonferroni-corrected per-parameter KS statistic) flags an invalid transform
\emph{or} a sampler that did not reach the posterior, both meaning ``do not
trust this.'' On a sticky horseshoe~\citep{carvalho2010} logistic regression ($D=18$) where
identity-plus-adaptation suffers 520 divergences (ESS/grad $6.2\times10^{-4}$),
every one of five proposer models found an SBC-certified reparam giving
$\ge\!10\times$ ESS/grad (up to $24\times$), with $0/25$ live proposals
SBC-rejected (block~H of Table~\ref{tab:master}). The gate earns its keep on a constructed
trap, $\theta=C\tanh(z)$: a smooth, non-singular bijection that samples easily but
targets a \emph{truncated, wrong} posterior. SBC rejects it while ESS, $\hat R$,
and the Jacobian smoke-test all green-light it.

\paragraph{The cleanest de-aliasing in the experiments.} Our hardest target is a
curved-ridge ``banana'' stack ($D=10$) whose optimal transform is a nonlinear
quadratic \emph{shear} $b\leftarrow b-c\,a^2$, outside the non-centering/affine
idiom set, with the curvature $c$ a single hidden nonlinear hole. Given only the
generative structure (curvature and widths withheld), all three proposer models
\emph{constructed the shear}. Vanilla autoresearch proposes the same shear at a
\emph{guessed} curvature (ELBO $\approx-74$); hybrid tunes the one hidden constant
to ELBO $-13$ (block~I of Table~\ref{tab:master}). The $+60$-nat win comes from
the inner loop tuning a single hidden parameter on an identical structure. The
affine case has a complementary
asymmetry: NUTS's linear metric cannot tune a nonlinear curvature, but once VI
finds $c$, NUTS samples the de-curved space at $+1.57$ decades ($\sim$$37\times$),
a NeuTra-style~\citep{hoffman2019} pipeline in which the two inner solvers compose.

\subsection{Analysis: when is the hybrid factorization worthwhile?}
\label{sec:scope}

Hybrid nested search cannot be an universal optimizer, by the No Free Lunch theorem
~\citep{wolpert1997, wolpert1996, schumacher2001, igel2004, droste2002}:
every search algorithm ties when averaged over all objectives, so any claim of superiority
is a claim about a particular class of objective functions the algorithm biases toward. Three features put our method outside the
strict black-box regime: the outer operator is
\emph{informed} (it is conditioned on rich, structured feedback, and a pretraining prior,
not only queried values); the $\tau\times\theta$ factorization acts as a reparameterization of the problem space; and the cost asymmetry (Sec.~\ref{sec:budget}) and
the tuned-vs-untuned signal avoid NFL's homogeneousquery accounting. The factorization thus
pays off on problems that are near decomposable (a good $\tau$ leaves a
well-behaved inner problem), carry structural prior knowledge (a near optimal $\tau$
resembles the training corpus), expose a non-negligible parametric gap $\E[\Delta]$ with holes that the inner solver can exploit, and are cheap to
evaluate; we predict hybrid search merely ties when these fail: the low parametric gap cases
we observe (\texttt{ellipsoid}, intra-cloud routing, the deceptive \texttt{schwefel}). The advantage vs. $\E[\Delta]$ trend
(e.g. Figure~\ref{fig:cloud-cleanup}c) plots where this prior places its mass.

\section{Related work}
\label{sec:related}

\paragraph{LLM-driven program and algorithm discovery.} Placing a frozen LLM in an
evolutionary loop to synthesize programs has driven discoveries in mathematics and
algorithms~\citep{romeraparedes2024, novikov2025}, with a wave of recent systems
pushing sample efficiency~\citep{lange2025}, recursive
self-improvement~\citep{zhang2025}, automated agent design~\citep{hu2024},
LLM-driven heuristic evolution~\citep{liu2024eoh}, and end-to-end
autoresearch~\citep{lu2024}. All treat the model as a single operator that
mutates program structure and embedded numeric constants jointly in one text turn.
We instead factor that operator by variable type; our orthogonality result
(Sec.~\ref{sec:policies}) shows the inner numerical tuner improves either a $(1{+}1)$
loop or a reflective prompt-evolution outer loop~\citep{agrawal2025}, and the
sample-efficiency goal of \citet{lange2025} is precisely what our budget argument
(Sec.~\ref{sec:budget}) secures, by exhausting each expensive proposal with a cheap
inner solver.

\paragraph{LLMs as optimizers, and prompt optimization.} A parallel line casts the
model itself as the optimizer over numeric or structured
spaces~\citep{yang2023}, as evolution strategies~\citep{lange2024}, to steer
Bayesian optimization~\citep{liu2024llambo}, or to search reward
code~\citep{ma2023}, asking the LLM to do the continuous search it is weakest at;
our de-aliasing analysis (Sec.~\ref{sec:dealias}) instead hands that subproblem to a
matched numerical solver. Reflective and ``textual-gradient'' prompt
methods~\citep{agrawal2025, yuksekgonul2024, guo2023} are, in our framework,
candidate outer operators rather than competitors: we adopt one (GEPA) and
show the inner tuner is orthogonal to and composes with it, closing the
numeric-constant gap none of them addresses.

\paragraph{Bilevel optimization, reparameterization, and automatic inference.} The
target Eq.~\eqref{eq:bilevel} is a bilevel program, studied for hyperparameter
optimization and meta-learning~\citep{franceschi2018} and solved at the inner level
by zero-order search~\citep{hansen2001}; our contribution is the assignment
of its two levels to operators matched by variable type. On the inference side,
non-centering~\citep{papaspiliopoulos2007}, normalizing-flow
transports~\citep{hoffman2019}, and gradient-based VI~\citep{kucukelbir2017} are the
target class of Sec.~\ref{sec:bayes}; closest is automatic
reparameterization~\citep{gorinova2020}, whereas we cast the transform as an LLM
structural search whose constants an inner solver tunes, certified reference-free
by SBC~\citep{talts2018}.

\section{Conclusion}

Factoring an LLM synthesis operator by variable type (structure to the model,
parameters to a nested numerical solver) removes the parametric aliasing that
makes joint search discard good structures with default parameters. The
advantage is predictable: it scales with the mean tuning gap and vanishes when the
model already guesses well. Our harness demonstrates strong results across the inner solver (CMA-ES, VI, NUTS),
the outer structural optimizer, and the proposer LLMs, across three different scientific domains. See Appendix~\ref{app:gallery} for a gallery of discovered code artifacts.

\paragraph{Limitations.} The manifest must be parseable, and mixed-integer holes need
more specialized solvers. The generalist CMA-ES degrades past a few dozen dimensions, so
$h(\tau)$ must be capped (high-dimensional holes call for a neural
inner operator, left for further work). A novelty ceiling remains, the method
is best used as a human-in-the-loop copilot rather than an autonomous
frontier system.

\bibliographystyle{colm2026_conference}
\bibliography{hybrid_search}

\begin{thebibliography}{35}
\providecommand{\natexlab}[1]{#1}
\providecommand{\url}[1]{\texttt{#1}}
\expandafter\ifx\csname urlstyle\endcsname\relax
  \providecommand{\doi}[1]{doi: #1}\else
  \providecommand{\doi}{doi: \begingroup \urlstyle{rm}\Url}\fi

\bibitem[Agrawal et~al.(2025)]{agrawal2025}
L.~Agrawal et~al.
\newblock {GEPA}: reflective prompt evolution can outperform reinforcement
  learning.
\newblock \emph{arXiv:2507.19457}, 2025.

\bibitem[Carvalho et~al.(2010)Carvalho, Polson, and Scott]{carvalho2010}
C.~M. Carvalho, N.~G. Polson, and J.~G. Scott.
\newblock The horseshoe estimator for sparse signals.
\newblock \emph{Biometrika}, 97\penalty0 (2):\penalty0 465--480, 2010.

\bibitem[Cheng et~al.(2024)Cheng, Nie, and Swaminathan]{cheng2024trace}
C.-A. Cheng, A.~Nie, and A.~Swaminathan.
\newblock Trace is the next {AutoDiff}: generative optimization with rich
  feedback, execution traces, and {LLMs}.
\newblock \emph{arXiv:2406.16218}, 2024.

\bibitem[Droste et~al.(2002)Droste, Jansen, and Wegener]{droste2002}
S.~Droste, T.~Jansen, and I.~Wegener.
\newblock Optimization with randomized search heuristics: the {(A)NFL} theorem,
  realistic scenarios, and difficult functions.
\newblock \emph{Theoretical Computer Science}, 287\penalty0 (1):\penalty0
  131--144, 2002.

\bibitem[Franceschi et~al.(2018)Franceschi, Frasconi, Salzo, Grazzi, and
  Pontil]{franceschi2018}
L.~Franceschi, P.~Frasconi, S.~Salzo, R.~Grazzi, and M.~Pontil.
\newblock Bilevel programming for hyperparameter optimization and
  meta-learning.
\newblock In \emph{ICML}, 2018.

\bibitem[Gallego(2026)]{gallego2026}
V.~Gallego.
\newblock Beyond scalar rewards: dense feedback for {LLM} policy synthesis in
  sequential social dilemmas.
\newblock \emph{arXiv:2603.19453}, 2026.

\bibitem[Gorinova et~al.(2020)Gorinova, Moore, and Hoffman]{gorinova2020}
M.~I. Gorinova, D.~Moore, and M.~D. Hoffman.
\newblock Automatic reparameterisation of probabilistic programs.
\newblock In \emph{ICML}, 2020.

\bibitem[Guo et~al.(2023)Guo, Wang, Guo, Li, et~al.]{guo2023}
Q.~Guo, R.~Wang, J.~Guo, B.~Li, et~al.
\newblock Connecting large language models with evolutionary algorithms yields
  powerful prompt optimizers.
\newblock \emph{arXiv:2309.08532}, 2023.

\bibitem[Hansen \& Ostermeier(2001)Hansen and Ostermeier]{hansen2001}
N.~Hansen and A.~Ostermeier.
\newblock Completely derandomized self-adaptation in evolution strategies.
\newblock \emph{Evolutionary Computation}, 9\penalty0 (2):\penalty0 159--195,
  2001.

\bibitem[Hoffman et~al.(2019)Hoffman, Sountsov, Dillon, Langmore, Tran, and
  Vasudevan]{hoffman2019}
M.~Hoffman, P.~Sountsov, J.~V. Dillon, I.~Langmore, D.~Tran, and S.~Vasudevan.
\newblock {NeuTra}-lizing bad geometry in {Hamiltonian Monte Carlo} using
  neural transport.
\newblock \emph{arXiv:1903.03704}, 2019.

\bibitem[Hoffman \& Gelman(2014)Hoffman and Gelman]{hoffman2014}
M.~D. Hoffman and A.~Gelman.
\newblock The {No-U-Turn} sampler: adaptively setting path lengths in
  {Hamiltonian Monte Carlo}.
\newblock \emph{Journal of Machine Learning Research}, 15:\penalty0 1593--1623,
  2014.

\bibitem[Hu et~al.(2024)Hu, Lu, and Clune]{hu2024}
S.~Hu, C.~Lu, and J.~Clune.
\newblock Automated design of agentic systems.
\newblock \emph{arXiv:2408.08435}, 2024.

\bibitem[Hughes et~al.(2018)Hughes, Leibo, Phillips, Tuyls, et~al.]{hughes2018}
E.~Hughes, J.~Z. Leibo, M.~Phillips, K.~Tuyls, et~al.
\newblock Inequity aversion improves cooperation in intertemporal social
  dilemmas.
\newblock In \emph{NeurIPS}, 2018.

\bibitem[Igel \& Toussaint(2004)Igel and Toussaint]{igel2004}
C.~Igel and M.~Toussaint.
\newblock A no-free-lunch theorem for non-uniform distributions of target
  functions.
\newblock \emph{Journal of Mathematical Modelling and Algorithms}, 3\penalty0
  (4):\penalty0 313--322, 2004.

\bibitem[Kucukelbir et~al.(2017)Kucukelbir, Tran, Ranganath, Gelman, and
  Blei]{kucukelbir2017}
A.~Kucukelbir, D.~Tran, R.~Ranganath, A.~Gelman, and D.~M. Blei.
\newblock Automatic differentiation variational inference.
\newblock \emph{Journal of Machine Learning Research}, 18:\penalty0 1--45,
  2017.

\bibitem[Lange et~al.(2024)Lange, Tian, and Tang]{lange2024}
R.~T. Lange, Y.~Tian, and Y.~Tang.
\newblock Large language models as evolution strategies.
\newblock \emph{arXiv:2402.18381}, 2024.

\bibitem[Lange et~al.(2025)Lange, Imajuku, and Cetin]{lange2025}
R.~T. Lange, Y.~Imajuku, and E.~Cetin.
\newblock {ShinkaEvolve}: towards open-ended and sample-efficient program
  evolution.
\newblock \emph{arXiv:2509.19349}, 2025.

\bibitem[Leibo et~al.(2017)Leibo, Zambaldi, Lanctot, Marecki, and
  Graepel]{leibo2017}
J.~Z. Leibo, V.~Zambaldi, M.~Lanctot, J.~Marecki, and T.~Graepel.
\newblock Multi-agent reinforcement learning in sequential social dilemmas.
\newblock In \emph{AAMAS}, 2017.

\bibitem[Liu et~al.(2024{\natexlab{a}})Liu, Tong, Yuan, Lin,
  et~al.]{liu2024eoh}
F.~Liu, X.~Tong, M.~Yuan, X.~Lin, et~al.
\newblock Evolution of heuristics: towards efficient automatic algorithm design
  using large language models.
\newblock \emph{arXiv:2401.02051}, 2024{\natexlab{a}}.

\bibitem[Liu et~al.(2024{\natexlab{b}})Liu, Astorga, Seedat, and van~der
  Schaar]{liu2024llambo}
T.~Liu, N.~Astorga, N.~Seedat, and M.~van~der Schaar.
\newblock Large language models to enhance {Bayesian} optimization.
\newblock \emph{arXiv:2402.03921}, 2024{\natexlab{b}}.

\bibitem[Lu et~al.(2024)Lu, Lu, Lange, Foerster, et~al.]{lu2024}
C.~Lu, C.~Lu, R.~T. Lange, J.~Foerster, et~al.
\newblock The {AI} {Scientist}: towards fully automated open-ended scientific
  discovery.
\newblock \emph{arXiv:2408.06292}, 2024.

\bibitem[Ma et~al.(2023)Ma, Liang, Wang, Huang, et~al.]{ma2023}
Y.~J. Ma, W.~Liang, G.~Wang, D.-A. Huang, et~al.
\newblock {Eureka}: human-level reward design via coding large language models.
\newblock \emph{arXiv:2310.12931}, 2023.

\bibitem[Novikov et~al.(2025)]{novikov2025}
A.~Novikov et~al.
\newblock {AlphaEvolve}: a coding agent for scientific and algorithmic
  discovery.
\newblock Technical report, 2025.

\bibitem[Papaspiliopoulos et~al.(2007)Papaspiliopoulos, Roberts, and
  Sk{\"o}ld]{papaspiliopoulos2007}
O.~Papaspiliopoulos, G.~O. Roberts, and M.~Sk{\"o}ld.
\newblock A general framework for the parametrization of hierarchical models.
\newblock \emph{Statistical Science}, 22\penalty0 (1):\penalty0 59--73, 2007.

\bibitem[Romera-Paredes et~al.(2024)]{romeraparedes2024}
B.~Romera-Paredes et~al.
\newblock Mathematical discoveries from program search with large language
  models.
\newblock \emph{Nature}, 625:\penalty0 468--475, 2024.

\bibitem[Schumacher et~al.(2001)Schumacher, Vose, and Whitley]{schumacher2001}
C.~Schumacher, M.~D. Vose, and L.~D. Whitley.
\newblock The no free lunch and problem description length.
\newblock In \emph{GECCO}, pp.\  565--570, 2001.

\bibitem[Talts et~al.(2018)Talts, Betancourt, Simpson, Vehtari, and
  Gelman]{talts2018}
S.~Talts, M.~Betancourt, D.~Simpson, A.~Vehtari, and A.~Gelman.
\newblock Validating {Bayesian} inference algorithms with simulation-based
  calibration.
\newblock \emph{arXiv:1804.06788}, 2018.

\bibitem[Vehtari et~al.(2021)Vehtari, Gelman, Simpson, Carpenter, and
  B{\"u}rkner]{vehtari2021}
A.~Vehtari, A.~Gelman, D.~Simpson, B.~Carpenter, and P.-C. B{\"u}rkner.
\newblock Rank-normalization, folding, and localization: an improved
  {$\widehat{R}$} for assessing convergence of {MCMC}.
\newblock \emph{Bayesian Analysis}, 16\penalty0 (2):\penalty0 667--718, 2021.

\bibitem[Wolpert(1996)]{wolpert1996}
D.~H. Wolpert.
\newblock The lack of a priori distinctions between learning algorithms.
\newblock \emph{Neural Computation}, 8\penalty0 (7):\penalty0 1341--1390, 1996.

\bibitem[Wolpert \& Macready(1997)Wolpert and Macready]{wolpert1997}
D.~H. Wolpert and W.~G. Macready.
\newblock No free lunch theorems for optimization.
\newblock \emph{IEEE Transactions on Evolutionary Computation}, 1\penalty0
  (1):\penalty0 67--82, 1997.

\bibitem[Wooders et~al.(2024)Wooders, Liu, Jain, Mo, Gonzalez, Liu, and
  Stoica]{wooders2024}
S.~Wooders, S.~Liu, P.~Jain, X.~Mo, J.~E. Gonzalez, V.~Liu, and I.~Stoica.
\newblock Cloudcast: high-throughput, cost-aware overlay multicast in the
  cloud.
\newblock In \emph{NSDI}, 2024.

\bibitem[Wu et~al.(2024)Wu, Chiang, Mao, Yang, Friedman, Shenker, and
  Stoica]{wu2024}
Z.~Wu, W.-L. Chiang, Z.~Mao, Z.~Yang, E.~Friedman, S.~Shenker, and I.~Stoica.
\newblock Can't be late: optimizing spot instance savings under deadlines.
\newblock In \emph{NSDI}, 2024.

\bibitem[Yang et~al.(2023)Yang, Wang, Lu, Liu, Le, Zhou, and Chen]{yang2023}
C.~Yang, X.~Wang, Y.~Lu, H.~Liu, Q.~V. Le, D.~Zhou, and X.~Chen.
\newblock Large language models as optimizers.
\newblock \emph{arXiv:2309.03409}, 2023.

\bibitem[Yuksekgonul et~al.(2024)Yuksekgonul, Bianchi, Boen, Liu,
  et~al.]{yuksekgonul2024}
M.~Yuksekgonul, F.~Bianchi, J.~Boen, S.~Liu, et~al.
\newblock {TextGrad}: automatic ``differentiation'' via text.
\newblock \emph{arXiv:2406.07496}, 2024.

\bibitem[Zhang et~al.(2025)Zhang, Hu, Lu, Lange, and Clune]{zhang2025}
J.~Zhang, S.~Hu, C.~Lu, R.~Lange, and J.~Clune.
\newblock Darwin {G\"odel} {Machine}: open-ended evolution of self-improving
  agents.
\newblock \emph{arXiv:2505.22954}, 2025.

\end{thebibliography}

\newpage
\appendix

\section{Proof of Proposition~\ref{prop:alias}}
\label{app:proof}

We recall the setup of \S\ref{sec:dealias}: the outer operator accepts by a rule
monotone in the observed score, the inner solver is $\varepsilon$-accurate
($|\widehat F(\tau)-F(\tau)|\le\varepsilon$), and
$\Delta(\tau)=F(\tau)-f(\tau,\bar\theta(\tau))\ge0$.

\begin{proof}
Write $g_{\mathrm{van}}(\tau)=f(\tau,\bar\theta(\tau))=F(\tau)-\Delta(\tau)$ and
$g_{\mathrm{hyb}}(\tau)=\widehat F(\tau)\in[F(\tau)-\varepsilon,F(\tau)+\varepsilon]$.
By monotone acceptance, every preference is a comparison of these scores.

\emph{(1)} Fix $\tau,\tau'$ and set $a=F(\tau)-F(\tau')$, $d=\Delta(\tau)-\Delta(\tau')$.
The vanilla score difference is $g_{\mathrm{van}}(\tau)-g_{\mathrm{van}}(\tau')=a-d$;
when $|a|>2\varepsilon$ the hybrid difference $g_{\mathrm{hyb}}(\tau)-g_{\mathrm{hyb}}(\tau')$
has the sign of $a$. The two agree iff $\operatorname{sign}(a-d)=\operatorname{sign}(a)$,
i.e.\ iff subtracting $\Delta(\tau)-\Delta(\tau')$ does not flip the sign of
$F(\tau)-F(\tau')$. They disagree (an \emph{inversion}) iff
$\operatorname{sign}(d)=\operatorname{sign}(a)$ and $|d|>|a|$.

\emph{(2)} Let $\tau^\star$ be a structural optimum, so $a=F(\tau^\star)-F(\tau')\ge0$.
Vanilla prefers the inferior $\tau'$ iff $g_{\mathrm{van}}(\tau')>g_{\mathrm{van}}(\tau^\star)$,
i.e.\ iff $\Delta(\tau^\star)-\Delta(\tau')>F(\tau^\star)-F(\tau')$, the stated condition.
For hybrid, $g_{\mathrm{hyb}}(\tau')-g_{\mathrm{hyb}}(\tau^\star)\le 2\varepsilon-a\le0$
whenever $a>2\varepsilon$, so hybrid never prefers a structure more than $2\varepsilon$ worse.

\emph{(3)} Model proposals as i.i.d.\ $\tau_1,\dots,\tau_N\sim\Mcal$ with $F\perp\Delta$
(the structural ceiling and the guess gap independent under $\Mcal$),
$\mu=\E[\Delta]$, $\sigma^2=\mathrm{Var}(\Delta)$. Vanilla deploys
$\hat\tau_v=\arg\max_i g_{\mathrm{van}}(\tau_i)$ at value $F(\hat\tau_v)-\Delta(\hat\tau_v)$;
hybrid deploys $\arg\max_i\widehat F(\tau_i)$ at value $\ge\max_i F(\tau_i)-\varepsilon$. Hence
\[
A_N=\big(\textstyle\max_i F(\tau_i)-F(\hat\tau_v)\big)+\Delta(\hat\tau_v)-\varepsilon,
\]
a sum of two non-negative terms (up to $\varepsilon$). The \emph{tuning} term
$\Delta(\hat\tau_v)\ge0$ is deployed on every accepted structure regardless of any
inversion; it is stochastically increasing in the scale of $\Delta$ and vanishes iff
$\Delta\equiv0$, so its expectation increases with $\mu$. The \emph{selection} term is
positive exactly on the inversions of (1); writing $d=\sigma d_0$ for standardized
symmetric $d_0$ independent of $a$, $\Pr[\text{inversion}\mid a]=\tfrac12\Pr[|d_0|>|a|/\sigma]$
is non-decreasing in $\sigma$, as is the expected structural regret
$\E[\max_i F(\tau_i)-F(\hat\tau_v)]$. Thus $\E[A_N]$ grows with both the mean and the
cross-structure variance of $\Delta$, and equals zero when $\Delta\equiv0$.
(The independence $F\perp\Delta$ is a simplifying assumption taken for a clean
two-term decomposition; allowing dependence---better structures systematically
carrying larger or smaller gaps---couples the terms but leaves the qualitative claim
intact, namely that $\E[A_N]$ increases in $\E[\Delta]$ and $\mathrm{Var}(\Delta)$,
which is what the experiments test.)
\end{proof}

\section{Additional results}
\label{app:additional}

This appendix collects the per-family configuration, interpretation, and figures
behind the main-text results; every numerical cell---all proposer models and
regimes---is consolidated in the comprehensive Table~\ref{tab:master}
(Appendix~\ref{app:master}), to which each subsection points.

\subsection{Meta-optimizers}
\label{app:meta}

Configuration: $N=3$ proposals per arm, inner CMA-ES budget $B_{\mathrm{in}}=100$,
direct-CMA-ES reference budget $2000$. ``Delivered'' vanilla loss is the minimum
untuned loss over the arm's proposals; delivered hybrid loss is the minimum tuned
loss. The full grid---all five objectives and three proposers, with the oracle
Direct CMA-ES reference---is block~A of Table~\ref{tab:master}.

On the convex \texttt{ellipsoid}, both arms reach machine-precision basins (the
$\E[\Delta]$ figure conflates ``bad guess'' with ``ambitious high-ceiling
proposal,'' so it can be large even where both arms solve the problem); the
practical difference is moot. On \texttt{schwefel} the global basin is far from
the start and the gradient is deceptive: GLM and Gemini hybrid improve over
vanilla and beat direct CMA-ES, but Opus ties (both arms trapped in the same
local minimum), and no arm reaches the global optimum.

\subsection{Cloud-systems policies}
\label{app:cloud}

Configuration: $N=3$ proposals per arm, $B_{\mathrm{in}}=100$, pure-CMA-ES budget
$120$. Delivered cost is the minimum over each arm's proposals (dollars; lower is
better); pure CMA-ES is model-independent (it tunes a fixed seed structure). Can't
Be Late (all three regimes) is block~B of Table~\ref{tab:master} and Cloudcast
block~C, where intra-cloud is the $\E[\Delta]\!\approx\!0$ tie control (cost-only
shortest paths are already near-optimal).

\begin{figure}[h]
\centering
\includegraphics[width=\textwidth]{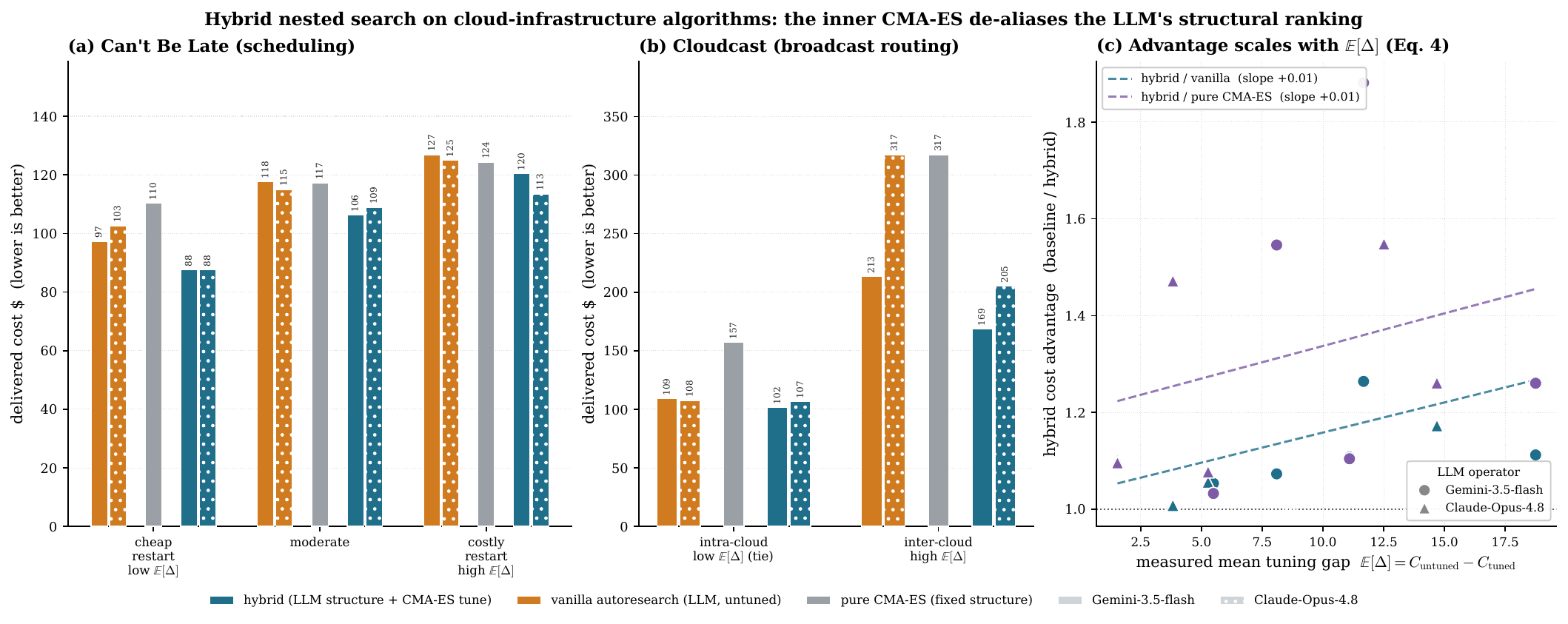}
\caption{\textbf{Cloud-systems policies (\S\ref{sec:policies}).} (a) Can't Be
Late and (b) Cloudcast delivered cost by arm and regime, two proposer models
(lower is better). (c) The discriminating prediction Eq.~\eqref{eq:eq4}: hybrid
cost advantage rises with the measured mean tuning gap $\E[\Delta]$; intra-cloud
routing ($\E[\Delta]\!\approx\!0$) ties.}
\label{fig:cloud-cleanup}
\end{figure}

Hybrid delivers the lowest cost in every cell. (We note the measured $\E[\Delta]$
ordering for Can't Be Late is not monotone in the engineered regime difficulty:
the LLM's untuned guesses happen to be poorest in the cheap-restart regime, so the
per-regime gap and the headline cost-advantage axes do not coincide for this
domain. Cloudcast and Cleanup carry the Eq.~\eqref{eq:eq4} trend in
Figure~\ref{fig:cloud-cleanup}c.)

\subsection{Cleanup}
\label{app:cleanup}

Welfare $U$ (higher is better), best-of-$N$ selection, $N=3$, $B_{\mathrm{in}}=100$;
pure CMA-ES (model-independent) is $0.93/0.58/0.26$ for light/moderate/heavy, and
the heavy hybrid ceiling is $0.582$ for all three models. The per-model, per-seed
welfare and the hybrid/vanilla ratios are block~D of Table~\ref{tab:master}.

\begin{figure}[h]
\centering
\includegraphics[width=\textwidth]{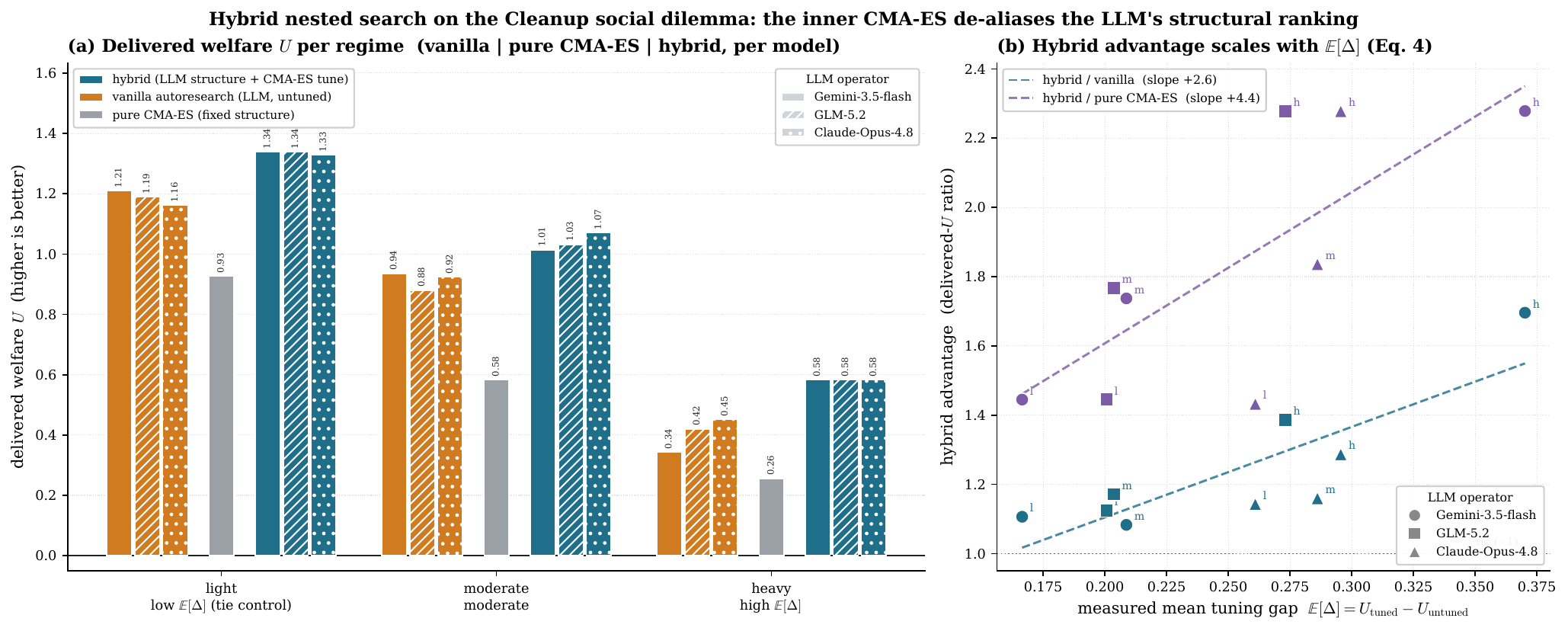}
\caption{Cleanup: (a) delivered welfare by arm, seed, and model; (b) hybrid
advantage vs.\ measured $\E[\Delta]$ (Eq.~\eqref{eq:eq4}), across models. Weaker
operators carry larger gaps and reap larger de-aliasing gains.}
\label{fig:cleanup}
\end{figure}

A weaker operator carries a larger gap and reaps a larger de-aliasing gain: Opus
is the best untuned guesser (highest vanilla on heavy, $0.45$) and shows the most
compressed advantage there ($1.29\times$ vs.\ Gemini's $1.70\times$), reading
Eq.~\eqref{eq:eq4} cleanly along the model-strength axis. Best-of-$N$ is the
conservative rule; hybrid wins 9/9 under last-iteration selection too, and is
markedly less sensitive to the rule than vanilla (the inner tuner repairs whatever
structure is current).

\subsection{GEPA orthogonality}
\label{app:gepa}

Replacing the (1+1) LLM outer loop with GEPA (reflective Pareto-frontier prompt
evolution, frozen reflection LM gemini-3.5-flash, $\le 30$ metric calls) while
keeping CMA-ES as the inner \Tune; ``GEPA-vanilla'' runs the proposed structures at
the LM's guessed constants (no inner loop), ``GEPA-hybrid'' tunes them. The
delivered numbers (cost for Can't Be Late and Cloudcast, welfare $U$ for Cleanup)
are block~E of Table~\ref{tab:master}: hybrid wins every non-negligible-gap regime,
and the two ties are the lowest-$\Delta$ regimes.

Compared with the (1+1) LLM runs (Appendices~\ref{app:cloud}, \ref{app:cleanup}),
the hybrid edge under GEPA points the same way at smaller magnitude: a strong
reflective optimizer recovers more of the parameter gain on its own, so the
\emph{marginal} value of the explicit inner tuner shrinks. It never reverses where
$\Delta$ is materially nonzero, and on the two low-$\Delta$ regimes where
GEPA-vanilla edges ahead on the headline metric, post-hoc tuning of its own pick
attributes the difference to structure-selection noise rather than to removing the
tuner.

\subsection{Bayesian inference: suite and self-test}
\label{app:bayes}

The reparameterization suite spans the matched/unmatched classes. $\Ccal^+$
(nonlinear geometry, hybrid wins): \texttt{funnel} ($D{=}10$, non-centering),
\texttt{eight\_schools} ($D{=}10$, hierarchical partial pooling, the aliasing
case). $\Ccal^-$ (affine controls, predicted ties under NUTS): \texttt{gauss\_ill}
(cond $10^4$, axis-aligned), \texttt{gauss\_rot} (same spectrum, rotated). Models
are unconstrained densities on $\R^D$ so the identity is the centered baseline. ESS
and $\hat R$ use the rank-normalized multi-chain estimator~\citep{vehtari2021} computed in-module.
ESS/grad is reported as the divergence-aware
$\mathrm{eff}=\mathrm{ESS}/\mathrm{grad}\cdot\exp(-40\cdot\mathrm{div\_rate})$ so a
high-ESS-but-divergent sampler is correctly demoted. The hybrid advantage by inner
solver---VI in nats (tuned ELBO vs.\ the untuned guess), NUTS in decades of
ESS/grad (hybrid vs.\ identity $+$ full warmup adaptation)---is blocks~F and~G of
Table~\ref{tab:master}.

For the hard horseshoe logistic regression and the curved banana (\S\ref{sec:bayes}),
SBC is the hard acceptance filter and the held-out gate. The SBC band is
Bonferroni-corrected across the $D$ parameters,
$\mathrm{band}=\sqrt{-\tfrac12\ln(\alpha_{\mathrm{fw}}/(2D))}/\sqrt{L}$ (family-wise
$\alpha_{\mathrm{fw}}=0.01$, $L$ replicates); without the correction valid
non-centering on the horseshoe is falsely rejected. SBC is cheap insurance that
rarely fires on live proposals from capable models (which propose only valid
idioms) but is the only signal that catches a fast-but-wrong transform such as the
$\tanh$ truncation trap.

\subsection{The hardest case studies}
\label{app:hard}

The reparameterization suite uses textbook geometries that NUTS adaptation often
already handles. We stress-test the loop on two harder targets, certified
\emph{reference-free} by SBC (\S\ref{sec:bayes}), to ask whether the gain survives
on models practitioners struggle with and on a transform that is \emph{not} a
retrieved idiom. Both keep compute low (Gaussian/logistic
likelihoods, $D\le18$), and both replace CMA-ES with VI-gradient and NUTS-warmup
inner solvers.

\paragraph{Horseshoe logistic regression ($D=18$).} A sticky model:
identity $+$ full NUTS adaptation suffers $520$ divergences (ESS/grad
$6.2\times10^{-4}$). The LLM searches for an SBC-certified reparam; the headline
is hybrid (best calibrated proposal) vs.\ pure-numerical (identity $+$ adaptation).
Every one of the five proposer models, frontier and cheap alike, found a certified
reparam giving $\ge\!10\times$ ESS/grad (block~H of Table~\ref{tab:master}), $0/25$
live proposals SBC-rejected. The frontier model buys $\sim$$2.3\times$ over the
cheap ones, short of an order of magnitude, and high reasoning effort gave no
benefit. The task rewards retrieving an apt idiom over deliberation. The five
SBC-certified reparams and their ESS/grad advantages (decades over the
$6.2\times10^{-4}$ baseline) are block~H of Table~\ref{tab:master}.

\paragraph{Curved-ridge ``banana'' ($D=10$).} The optimal transform is a
nonlinear quadratic shear $b\!\leftarrow\!b-c\,a^2$ whose curvature $c$ is a single
hidden hole, sitting \emph{outside} the non-centering/affine/sinh idiom set and
forming the canonical failure case for Gaussian VI. The vanilla-numerical identity
baseline is ELBO $-174$, NUTS ESS/grad $4.1\times10^{-3}$. All three proposers
constructed the shear; vanilla-AR proposes it at a \emph{guessed} curvature, hybrid
tunes the one constant (block~I of Table~\ref{tab:master}). The de-aliasing win is
$\sim$$+60$ nats from tuning a single hidden parameter on an identical structure,
the cleanest demonstration in the project. The complementary asymmetry: NUTS's
linear metric cannot tune the curvature, but once VI finds it, NUTS samples the
de-curved space at $+1.57$ decades ($\sim$$37\times$).

\section{A gallery of discovered artifacts}
\label{app:gallery}

The factorization makes a concrete prediction about \emph{what} the LLM
contributes: an apt algorithmic skeleton whose load-bearing constants it cannot
place. The eight artifacts below (two optimizers, three policies, three
reparameterizations) each exhibit the pattern: a recognizable structural idea
that is mediocre or divergent at the model's guessed constants and excellent once
the inner solver finds them. We reproduce the discovered code (lightly elided);
the inner solver set the tuned constants, the model only the structure.

\paragraph{An optimizer: Nesterov with cosine schedule and gradient clipping.}
GLM-5.2 (hybrid arm) discovered this one on the hidden \texttt{rosenbrock} valley
(Listing~\ref{lst:opt}). The structure is textbook: Nesterov lookahead, a
warmup-then-cosine learning-rate schedule, $L_2$ gradient clipping. Its default
constants stall in the curved valley (loss $1.51$). CMA-ES pushes the learning
rate \emph{up} $120\times$ ($0.05\!\to\!6.0$) and the clip threshold \emph{down}
$88\times$ ($1.0\!\to\!0.011$): an enormous step rate that stays stable only
because every gradient is clipped to a tiny fixed norm, turning the method into a
large-step normalized crawl along the valley floor (tuned loss
$2.0\times10^{-6}$). This knife-edge is the regime vanilla discards.

\begin{lstlisting}[caption={Discovered optimizer \texttt{nesterov\_cosine\_clip} (GLM-5.2, \texttt{rosenbrock}). Tuned: \texttt{lr}=6.0, \texttt{momentum}=0.886, \texttt{clip}=0.0114, \texttt{warmup\_frac}=0.003.},label={lst:opt}]
def custom_optimizer(grad_fn, x, y, steps, params):
    lr = params['lr']; momentum = params['momentum']; clip = params['clip']
    eta_min = params['eta_min']; warmup_frac = params['warmup_frac']
    vx = 0.0; vy = 0.0
    warmup_steps = int(warmup_frac * steps)
    for i in range(steps):
        if i < warmup_steps and warmup_steps > 0:
            cur_lr = lr * (i + 1) / warmup_steps
        else:
            t = (i - warmup_steps) / max(steps - warmup_steps - 1, 1)
            cur_lr = eta_min + 0.5 * (lr - eta_min) * (1.0 + math.cos(math.pi * t))
        look_x = x - momentum * vx; look_y = y - momentum * vy   # Nesterov lookahead
        gx, gy = grad_fn(look_x, look_y)
        gnorm = math.sqrt(gx * gx + gy * gy)
        if gnorm > clip and gnorm > 0.0:                          # L2 gradient clip
            s = clip / gnorm; gx *= s; gy *= s
        vx = momentum * vx + gx; vy = momentum * vy + gy
        x -= cur_lr * vx; y -= cur_lr * vy
    return x, y
\end{lstlisting}

\paragraph{A second optimizer: restarts for a multimodal lattice.} On the
multimodal \texttt{rastrigin} lattice, Opus 4.8 wrapped a Nesterov heavy-ball
method in \emph{periodic random restarts} (Listing~\ref{lst:opt2}): every few
dozen steps it kicks the iterate by a random perturbation and zeroes the
velocity --- the one piece of \emph{globally} aware structure that plain CMA-ES
on the objective lacks (Figure~\ref{fig:meta}b). It is the textbook
Proposition~\ref{prop:alias} inversion: at its guessed constants this is the
\emph{worst} of the model's three proposals (untuned loss $39.6$, a wild kick
that bounces the iterate out of every basin), yet CMA-ES tunes the restart
period ($\approx\!30$ of $1000$ steps), the kick scale, and the slow LR decay
into the \emph{best} of the three (tuned $1.57$), beating both vanilla's best
untuned pick ($10.4$) and the oracle Direct CMA-ES ($9.95$). Vanilla, ranking by
untuned score, discards exactly the structure that wins.

\begin{lstlisting}[caption={Discovered optimizer with periodic random restarts (Opus 4.8, \texttt{rastrigin}). Tuned: \texttt{restart\_period}$\approx$30, \texttt{pert}=1.98, \texttt{lr}=0.104, \texttt{mom}=0.838, \texttt{lr\_decay}=0.9992.},label={lst:opt2}]
def custom_optimizer(grad_fn, x, y, steps, params):
    lr = params['lr']; mom = params['mom']; clip = params['clip']
    pert = params['pert']; pert_decay = params['pert_decay']
    restart_period = max(1, int(params['restart_period']))
    rng = np.random.RandomState(int(abs(params['seed']) * 1000) % 100000 + 1)
    vx = 0.0; vy = 0.0; cur_lr = lr
    for i in range(steps):
        gx, gy = grad_fn(x + mom * vx, y + mom * vy)              # Nesterov lookahead
        gn = math.sqrt(gx * gx + gy * gy) + 1e-12
        if gn > clip: gx *= clip / gn; gy *= clip / gn            # L2 gradient clip
        vx = mom * vx - cur_lr * gx; vy = mom * vy - cur_lr * gy
        x += vx; y += vy
        if (i + 1) % restart_period == 0:                         # periodic random restart
            scale = pert * (pert_decay ** (i / restart_period))
            x += scale * rng.randn(); y += scale * rng.randn()
            vx = 0.0; vy = 0.0                                    # zero the velocity
        cur_lr *= params['lr_decay']
    return x, y
\end{lstlisting}

\paragraph{A policy: a stateless hysteresis controller.} The most inventive
artifact in the study came from Claude Opus 4.8 (Cleanup, hybrid arm;
Listing~\ref{lst:policy}). Every other proposal is a memoryless map
$\mathrm{waste\_ratio}\!\to\!\mathrm{cleaners}$; Opus instead reached for
\emph{hysteresis}: commit to a high cleaner count once pollution crosses an upper
threshold and relax only below a \emph{lower} one, to stop the team chattering
cleaners on and off around a single breakpoint. The policy API is \emph{stateless}
(a pure function of the current observation, no shared memory), so classic
hysteresis is impossible; the model reconstructs the missing latch from the
observable itself (which side of the band midpoint the waste ratio sits on).
Because every agent computes the identical quantities, they all infer the same
committed regime, which coordinates the team with zero communication through
common knowledge. Its value hinges on threshold placement: CMA-ES set
\texttt{hi\_thr}$=0.299$, right \emph{on} the hidden apple-death cliff at $0.30$
that the model never sees, lifting welfare from $1.131$ (untuned) to $1.328$.

\begin{lstlisting}[caption={Discovered policy \texttt{hysteresis\_two\_threshold\_band} (Opus 4.8, Cleanup). Tuned: \texttt{lo\_thr}=0.073, \texttt{hi\_thr}=0.299, \texttt{low\_frac}=0.110, \texttt{mid\_frac}=0.606, \texttt{high\_frac}=0.935.},label={lst:policy}]
def policy(env, agent_id, params):
    wr = waste_ratio(env); n = env.n_agents
    lo_thr = params['lo_thr']; hi_thr = params['hi_thr']
    if hi_thr < lo_thr: lo_thr, hi_thr = hi_thr, lo_thr
    low_frac = params['low_frac']; high_frac = params['high_frac']; mid_frac = params['mid_frac']
    mid = 0.5 * (lo_thr + hi_thr)
    if wr >= hi_thr:    frac = high_frac                  # hard commit: high regime
    elif wr <= lo_thr:  frac = low_frac                   # hard commit: low regime
    else:                                                 # infer latch from band side
        if wr >= mid: frac = mid_frac + (high_frac - mid_frac) * ((wr - mid) / max(hi_thr - mid, 1e-6))
        else:         frac = low_frac + (mid_frac - low_frac) * ((wr - lo_thr) / max(mid - lo_thr, 1e-6))
    n_cleaners = max(0.0, min(float(n), frac * n))
    roles = assign_roles(env, int(round(n_cleaners)))
    return clean_action(env, agent_id) if roles[agent_id] == "clean" else harvest_action(env, agent_id)
\end{lstlisting}

\paragraph{A policy: a risk-aware spot/on-demand scheduler.} On Can't Be Late
(\S\ref{sec:policies}, costly-restart regime), Gemini 3.5 Flash proposed the
slack-gated scheduler of Listing~\ref{lst:cbl}. Cheap spot capacity is
preemptible, reliable on-demand is expensive, and the artifact maps the per-step
state (slack to deadline, restart overhead, remaining work) to one of three
actions --- take spot, take on-demand, or \emph{wait} for spot to return ---
through slack thresholds that grow with both the restart overhead \emph{and} the
remaining work, plus a finish-line lock-in and a hysteresis bonus that suppresses
costly toggling. The model exposes those buffers as holes at cautious guesses
(untuned \$$127.6$); CMA-ES drives the spot buffer sharply negative and inflates
the finish-line and hysteresis margins, delivering \$$120.5$, under both vanilla
(\$$126.9$) and the deployable pure-CMA-ES baseline (\$$124.4$).

\begin{lstlisting}[caption={Discovered scheduler \texttt{CumulativeRiskAwareSlackPolicy} (Gemini 3.5 Flash, Can't Be Late, costly restarts). Actions: $0$ wait, $1$ spot, $2$ on-demand. Tuned (selected): \texttt{spot\_offset}=$-3480$, \texttt{finish\_spot\_buffer}=$11337$, \texttt{hysteresis}=$6714$.},label={lst:cbl}]
def decide(obs, params):
    slack = obs["slack"]; ro = obs["restart_overhead"]
    rem = obs["remaining_task_time"]; has_spot = obs["has_spot"]
    last = int(obs["last_cluster_type"])
    if obs["remaining_restart_overhead"] > 0:               # mid-restart: stay put
        if last == 2: return 2
        if last == 1 and has_spot: return 1
    near_done = rem < params["finish_coef"] * ro + params["finish_offset"]   # finish-line lock-in
    spot_thr = params["spot_coef"] * ro + params["spot_task_coef"] * rem + params["spot_offset"]
    wait_thr = params["wait_coef"] * ro + params["wait_task_coef"] * rem + params["wait_offset"]
    if near_done: spot_thr += params["finish_spot_buffer"]
    if has_spot:
        if last == 2: spot_thr += params["hysteresis"]      # don't toggle off on-demand cheaply
        return 1 if slack > spot_thr else 2
    if last == 2: wait_thr += params["hysteresis"]
    return 0 if slack > wait_thr else 2                     # 0 = wait for spot to return
\end{lstlisting}

\paragraph{A policy: a self-sharing Steiner broadcast tree.} On Cloudcast
inter-cloud routing, Gemini 3.5 Flash built the graph algorithm of
Listing~\ref{lst:cc}: a greedy per-partition Steiner-tree approximation in which
each destination is routed by shortest path and \emph{the edges it uses are then
discounted}, so later destinations in the same partition reuse them, growing a
shared broadcast tree rather than independent paths; a deterministic sinusoidal
perturbation diversifies the parallel partitions. CMA-ES tunes the
cost/throughput blend ($\alpha,\beta$), the sharing discount, and the
perturbation scale. Hybrid delivers \$$168.7$, against vanilla \$$213.3$ and ---
the headline gap --- $1.9\times$ below pure CMA-ES (\$$317$) tuning a fixed
topology: here structural discovery, not constant-tuning, carries the domain.

\begin{lstlisting}[caption={Discovered router \texttt{GreedySteinerPartitionSharing} (Gemini 3.5 Flash, Cloudcast inter-cloud). Tuned: \texttt{alpha}=\sci{4.1}{-4}, \texttt{beta}=2.06, \texttt{sharing\_factor}=0.018, \texttt{perturb\_scale}=0.042.},label={lst:cc}]
def search_algorithm(src, dsts, G, num_partitions, params):
    idx = {n: i for i, n in enumerate(G.nodes())}
    base = {(u, v): G[u][v]['cost'] + params['alpha'] / G[u][v]['throughput'] ** params['beta']
            for u, v in G.edges()}                              # cost + throughput penalty
    dist = nx.single_source_dijkstra_path_length(G, src, weight=lambda u, v, d: base[(u, v)])
    order = sorted(dsts, key=lambda d: dist.get(d, 1e9), reverse=params['sort_reverse'] > 0.5)
    bc = BroadCastTopology(src, dsts, num_partitions)
    for j in range(num_partitions):                            # one shared tree per partition
        w = {e: base[e] * (1 + params['perturb_scale'] *
             math.sin(idx[e[0]] * 17.23 + idx[e[1]] * 43.19 + j * 97.43)) for e in G.edges()}
        for dst in order:
            path = nx.shortest_path(G, src, dst, weight=lambda u, v, d: w[(u, v)])
            for u, v in zip(path[:-1], path[1:]):
                bc.append_dst_partition_path(dst, j, [u, v, G[u][v]])
                w[(u, v)] *= params['sharing_factor']          # discount used edges -> reuse
    return bc
\end{lstlisting}

\paragraph{A reparameterization: a nonlinear quadratic shear.} On the curved-ridge
``banana'' (\S\ref{sec:bayes}), all three proposer models independently
constructed the same non-textbook transform: a nonlinear quadratic shear that
straightens the ridge (Listing~\ref{lst:banana}, the minimal three-hole form of
the family). The curvature \texttt{curv} is the single hidden hole; VI tuned it to
$0.937$ and the ridge width \texttt{b\_width} to $0.300$, recovering the
generative constants ($1.0$ and $0.3$). Vanilla proposes the same shear at a
\emph{guessed} curvature (ELBO $\approx-74$); the inner loop tunes the one
constant to ELBO $-66$ (and $-13$ for a per-coordinate variant), against the
identity baseline $-174$.

\begin{lstlisting}[caption={Discovered reparameterization (curved-ridge banana; family found by Opus 4.8, Gemini 3.1 Pro, Gemini 3.5 Flash). Manifest: \texttt{curv} guess 1.0 (range $-5$ to $5$). VI-tuned: \texttt{curv}=0.937, \texttt{a\_scale}=0.527, \texttt{b\_width}=0.300.},label={lst:banana}]
def reparam(z, params):
    K = 5
    a = params['a_scale'] * z[:K]
    b = params['curv'] * a**2 + params['b_width'] * z[K:]   # nonlinear shear straightens the ridge
    return jnp.concatenate([a, b])
\end{lstlisting}

\paragraph{A reparameterization: partial non-centering for a sparse funnel.} On
the sticky horseshoe logistic regression ($D=18$), Gemini 3.1 Pro found the most
effective SBC-certified transform (Listing~\ref{lst:horseshoe}). It is the
cheapest structure that works: an affine map plus a single
\emph{partial-non-centering} knob \texttt{nc} that scales the coefficient block by
a tunable power of the global$\times$local shrinkage. NUTS adaptation alone leaves
$520$ divergences (ESS/grad $6.2\times10^{-4}$); this transform reaches ESS/grad
$1.48\times10^{-2}$ ($\sim$$24\times$, $73$ divergences) and passes SBC, ahead of
heavier Cauchy / sinh-arcsinh competitors.

\begin{lstlisting}[caption={Discovered reparameterization \texttt{PartialNCP\_Affine} (Gemini 3.1 Pro, horseshoe logistic regression). Key hole \texttt{nc} interpolates centered ($0$) to non-centered ($1$); delivered at \texttt{nc}=0.5, unit affine scales.},label={lst:horseshoe}]
def reparam(z, params):
    theta_0 = z[0:1] * params['s_int'] + params['m_int']
    theta_1 = z[1:2] * params['s_glob'] + params['m_glob']
    theta_2 = z[2:10] * params['s_loc'] + params['m_loc']
    eff_scale = jnp.exp(params['nc'] * (theta_1[0] + theta_2))   # partial non-centering
    theta_3 = z[10:18] * eff_scale * params['s_coef'] + params['m_coef']
    return jnp.concatenate([theta_0, theta_1, theta_2, theta_3])
\end{lstlisting}

\paragraph{A reparameterization: a learned affine decorrelation.} The affine
controls make the complementary point to the banana's nonlinear shear. On
\texttt{gauss\_rot}, a rotated ill-conditioned Gaussian, Gemini 3.5 Flash
proposed the dense affine map of Listing~\ref{lst:affine}: per-coordinate shifts
and log-scales plus a rank-1 strictly-lower-triangular mixing term (a
cumulative-sum coupling) that captures an arbitrary-direction rotation in $D=10$
with $38$ holes. At the guessed constants the transform is \emph{worse than doing
nothing} (untuned ELBO $-33.4$, below the identity guide's $-28.5$): the aliasing
trap again. VI tunes all $38$ holes to ELBO $+4.30$, a $+37.7$-nat gain over
vanilla, and the headline VI-vs-NUTS asymmetry follows --- a fixed standard guide
cannot rescale or rotate without the transform (VI gains enormously), yet the
\emph{same} transform barely moves NUTS, whose warmup mass matrix already is a
linear preconditioner. This is Eq.~\eqref{eq:eq4} read along the linear/nonlinear
axis: redundant with the inner solver's competence, so $\Delta\!\approx\!0$ under
NUTS; orthogonal to it, so a large $\Delta$ under VI.

\begin{lstlisting}[caption={Discovered reparameterization \texttt{Rank1LowerTriangularAffine} (Gemini 3.5 Flash, \texttt{gauss\_rot}, VI). The $38$ holes are a per-coordinate shift \texttt{mu} and log-scale \texttt{s}, and a rank-1 lower-triangular coupling \texttt{(a, b)}; literal unpacking elided.},label={lst:affine}]
def reparam(z, params):
    mu = jnp.array([params[f'mu{i}'] for i in range(10)])   # per-coordinate shift
    s  = jnp.array([params[f's{i}']  for i in range(10)])   # per-coordinate log-scale
    a  = jnp.array([params[f'a{i}']  for i in range(9)])
    b  = jnp.array([params[f'b{i}']  for i in range(9)])
    w = jnp.concatenate([jnp.zeros(1), jnp.cumsum(b * z[:-1])])   # rank-1 lower-triangular mixing
    a_pad = jnp.concatenate([jnp.zeros(1), a])
    return mu + jnp.exp(s) * (z + a_pad * w)                # affine: shift + log-scale + correlation
\end{lstlisting}

Across all eight, the LLM supplies the \emph{idea} (accelerated descent with
clipping, globally aware restarts, stateful control reconstructed from common
knowledge, risk-aware spot scheduling, a self-sharing broadcast tree, a quadratic
shear, partial non-centering, a learned affine decorrelation) and the inner
solver supplies the \emph{constants} that decide whether the idea works. This is
the factorization made tangible.

\section{Consolidated results: every cell}
\label{app:master}

Table~\ref{tab:master} is the single source for every experimental cell --- all
proposer models, regimes, and inner/outer solvers from Appendix~\ref{app:additional}
--- in one schema: \textbf{vanilla} / \textbf{numerical}-only / \textbf{hybrid},
the mean tuning gap $\E[\Delta]$, and the hybrid advantage. \textbf{Bold} marks
the better of vanilla/hybrid; ``\nad'' an arm not run for that arm-definition (the
VI arm reports hybrid vs.\ vanilla, the NUTS / horseshoe arms hybrid vs.\
pure-numerical). The numerical arm is an \emph{oracle} (Direct CMA-ES) only in
block~A; elsewhere it is a deployable no-structure baseline. Shading of the
$\E[\Delta]$ and Hyb.\ adv.\ columns is a per-group heatmap (darker $=$ larger):
matching gradients down the two columns are the visual signature of
advantage\,$\propto\E[\Delta]$ (Eq.~\ref{eq:eq4}). Meta (block~A) is shaded per
model on a log scale; the policy blocks B--D per block on a linear scale ---
Cloudcast (C) and Cleanup (D) track cleanly, while Can't Be Late (B), whose gap
is non-monotone in regime, intentionally does not align. Advantage units are
per-block as stated in each header.

{\scriptsize
\setlength{\tabcolsep}{3pt}
\renewcommand{\arraystretch}{1.12}
\begin{longtable}{@{}llccccc@{}}
\caption{Consolidated experimental results (all cells).}\label{tab:master}\\
\toprule
Task & Model & Vanilla & Numerical & Hybrid & $\E[\Delta]$ & Hyb.\ adv. \\
\midrule
\endfirsthead
\multicolumn{7}{@{}l}{\scriptsize\itshape Table~\ref{tab:master} continued}\\
\toprule
Task & Model & Vanilla & Numerical & Hybrid & $\E[\Delta]$ & Hyb.\ adv. \\
\midrule
\endhead
\midrule
\multicolumn{7}{r}{\scriptsize\itshape continued on next page}\\
\endfoot
\bottomrule
\endlastfoot
\famrow{Family 1\quad Meta-optimizers on closed-form objectives}\\
\addlinespace[1pt]
\blkrow{\textbf{A. Meta-optimizers}\dt outer: (1+1) LLM\dt inner: CMA-ES\dt metric: final loss \dnar\dt numerical: Direct CMA-ES (\emph{oracle} ref.)\dt adv: loss reduction van$-$hyb}\\
\cmidrule(l{0.2em}r{0.2em}){1-7}
rosenbrock & Opus 4.8         & \sci{5.7}{-6}  & \sci{4.2}{-13} & \bsci{2.4}{-25}  & \hc{10}{$0.27$}  & \hc{10}{\sci{+5.7}{-6}} \\
ellipsoid  & Opus 4.8         & \sci{9.3}{-9}  & \sci{8.6}{-14} & \bsci{3.8}{-106} & \hc{44}{$5.9$}   & \hc{10}{tie$^{\dagger}$} \\
rastrigin  & Opus 4.8         & $10.4$         & $9.95$         & $\mathbf{1.57}$  & \hc{51}{$10.8$}  & \hc{57}{$+8.8$} \\
ackley     & Opus 4.8         & $10.8$         & \sci{4.4}{-12} & \bsci{4.1}{-13}  & \hc{52}{$12.3$}  & \hc{58}{$+10.8$} \\
schwefel   & Opus 4.8         & $405$          & $358$          & $405$            & \hc{58}{$20.4$}  & \hc{10}{tie} \\
\addlinespace[1.5pt]
rosenbrock & GLM-5.2          & \sci{2.1}{-4}  & \sci{4.2}{-13} & \bsci{1.5}{-12}  & \hc{10}{$0.28$}  & \hc{10}{\sci{+2.1}{-4}} \\
ellipsoid  & GLM-5.2          & \bsci{2.3}{-40}& \sci{8.6}{-14} & \sci{3.0}{-28}   & \hc{23}{$1.1$}   & \hc{10}{tie$^{\dagger}$} \\
rastrigin  & GLM-5.2          & $19.2$         & $9.95$         & $\mathbf{1.22}$  & \hc{45}{$11.0$}  & \hc{51}{$+18.0$} \\
ackley     & GLM-5.2          & $13.1$         & \sci{4.4}{-12} & \bsci{3.3}{-9}   & \hc{45}{$11.4$}  & \hc{49}{$+13.1$} \\
schwefel   & GLM-5.2          & $405$          & $358$          & $\mathbf{297}$   & \hc{58}{$43.0$}  & \hc{58}{$+108$} \\
\addlinespace[1.5pt]
rosenbrock & Gemini 3.5 Flash & \sci{8.6}{-6}  & \sci{4.2}{-13} & \bsci{2.7}{-22}  & \hc{10}{$0.003$} & \hc{10}{\sci{+8.6}{-6}} \\
ellipsoid  & Gemini 3.5 Flash & \sci{4.5}{-38} & \sci{8.6}{-14} & $\mathbf{0.0}$   & \hc{58}{$418.6$} & \hc{10}{tie$^{\dagger}$} \\
rastrigin  & Gemini 3.5 Flash & $19.2$         & $9.95$         & $\mathbf{0.33}$  & \hc{43}{$10.6$}  & \hc{49}{$+18.9$} \\
ackley     & Gemini 3.5 Flash & $13.9$         & \sci{4.4}{-12} & \bsci{5.7}{-7}   & \hc{44}{$14.0$}  & \hc{48}{$+13.9$} \\
schwefel   & Gemini 3.5 Flash & $405$          & $358$          & $\mathbf{227}$   & \hc{52}{$92.6$}  & \hc{58}{$+178$} \\
\addlinespace[3pt]
\famrow{Family 2\quad Executable policies (systems \& social dilemmas)}\\
\addlinespace[1pt]
\blkrow{\textbf{B. Can't Be Late} (spot/on-demand scheduling)\dt (1+1) LLM / CMA-ES\dt metric: cost \$ \dnar\dt numerical: Pure CMA-ES\dt adv: \% vs.\ vanilla}\\
\cmidrule(l{0.2em}r{0.2em}){1-7}
cheap restart   & Gemini & $97.4$  & $110.3$ & $\mathbf{87.6}$  & \hc{58}{$18.8$} & \hc{35}{$10.1\%$} \\
cheap restart   & Opus   & $102.6$ & $110.3$ & $\mathbf{87.5}$  & \hc{47}{$14.7$} & \hc{58}{$14.7\%$} \\
moderate        & Gemini & $117.7$ & $117.3$ & $\mathbf{106.2}$ & \hc{37}{$11.1$} & \hc{34}{$9.8\%$} \\
moderate        & Opus   & $115.0$ & $117.3$ & $\mathbf{108.9}$ & \hc{21}{$5.3$}  & \hc{12}{$5.3\%$} \\
costly restart  & Gemini & $126.9$ & $124.4$ & $\mathbf{120.5}$ & \hc{21}{$5.5$}  & \hc{10}{$5.0\%$} \\
costly restart  & Opus   & $125.0$ & $124.4$ & $\mathbf{113.5}$ & \hc{10}{$1.5$}  & \hc{31}{$9.2\%$} \\
\addlinespace[2pt]
\blkrow{\textbf{C. Cloudcast} (multi-cloud broadcast routing)\dt (1+1) LLM / CMA-ES\dt metric: cost \$ \dnar\dt numerical: Pure CMA-ES\dt adv: \% vs.\ vanilla\dt intra $=\E[\Delta]{\approx}0$ control}\\
\cmidrule(l{0.2em}r{0.2em}){1-7}
intra-cloud & Gemini & $109.0$ & $157.1$ & $\mathbf{101.7}$ & \hc{34}{$8.1$}  & \hc{18}{$6.7\%$} \\
intra-cloud & Opus   & $107.6$ & $157.1$ & $\mathbf{106.8}$ & \hc{10}{$3.8$}  & \hc{10}{$0.7\%$} \\
inter-cloud & Gemini & $213.3$ & $317.4$ & $\mathbf{168.7}$ & \hc{54}{$11.7$} & \hc{38}{$20.9\%$} \\
inter-cloud & Opus   & $317.4$ & $317.4$ & $\mathbf{205.1}$ & \hc{58}{$12.5$} & \hc{58}{$35.4\%$} \\
\addlinespace[2pt]
\blkrow{\textbf{D. Cleanup} (sequential social dilemma)\dt (1+1) LLM / CMA-ES\dt metric: welfare $U$ \upar\dt numerical: Pure CMA-ES (light/mod/heavy $=0.93/0.58/0.26$)\dt adv: hyb/van}\\
\cmidrule(l{0.2em}r{0.2em}){1-7}
light    & Gemini 3.5 Flash & $1.21$ & $0.93$ & $\mathbf{1.34}$ & \hc{10}{$0.17$} & \hc{12}{$1.11\times$} \\
moderate & Gemini 3.5 Flash & $0.94$ & $0.58$ & $\mathbf{1.01}$ & \hc{20}{$0.21$} & \hc{10}{$1.08\times$} \\
heavy    & Gemini 3.5 Flash & $0.34$ & $0.26$ & $\mathbf{0.58}$ & \hc{58}{$0.37$} & \hc{58}{$1.70\times$} \\
\addlinespace[1.5pt]
light    & GLM-5.2          & $1.19$ & $0.93$ & $\mathbf{1.34}$ & \hc{17}{$0.20$} & \hc{14}{$1.13\times$} \\
moderate & GLM-5.2          & $0.88$ & $0.58$ & $\mathbf{1.03}$ & \hc{17}{$0.20$} & \hc{17}{$1.17\times$} \\
heavy    & GLM-5.2          & $0.42$ & $0.26$ & $\mathbf{0.58}$ & \hc{34}{$0.27$} & \hc{34}{$1.39\times$} \\
\addlinespace[1.5pt]
light    & Opus 4.8         & $1.16$ & $0.93$ & $\mathbf{1.33}$ & \hc{32}{$0.26$} & \hc{15}{$1.14\times$} \\
moderate & Opus 4.8         & $0.92$ & $0.58$ & $\mathbf{1.07}$ & \hc{39}{$0.29$} & \hc{16}{$1.16\times$} \\
heavy    & Opus 4.8         & $0.45$ & $0.26$ & $\mathbf{0.58}$ & \hc{41}{$0.30$} & \hc{26}{$1.29\times$} \\
\addlinespace[2pt]
\blkrow{\textbf{E. GEPA as outer operator} (orthogonality)\dt outer: GEPA\dt inner: CMA-ES\dt LM: Gemini 3.5 Flash\dt numerical: Pure CMA-ES\dt adv: signed \% (cost or $U$)}\\
\cmidrule(l{0.2em}r{0.2em}){1-7}
Can't Be Late\,\dt cheap (\$\,\dnar)    & GEPA & $\mathbf{65.9}$  & $110.3$ & $66.0$           & \nad & $-0.2\%$ tie \\
Can't Be Late\,\dt moderate (\$\,\dnar) & GEPA & $88.2$           & $117.3$ & $\mathbf{81.5}$  & \nad & $+7.6\%$ \\
Can't Be Late\,\dt costly (\$\,\dnar)   & GEPA & $94.8$           & $124.4$ & $\mathbf{91.7}$  & \nad & $+3.3\%$ \\
Cleanup\,\dt light ($U$\,\upar)         & GEPA & $1.21$           & $0.93$  & $\mathbf{1.34}$  & \nad & $+10.7\%$ \\
Cleanup\,\dt moderate ($U$\,\upar)      & GEPA & $0.97$           & $0.58$  & $\mathbf{1.07}$  & \nad & $+10.6\%$ \\
Cleanup\,\dt heavy ($U$\,\upar)         & GEPA & $0.51$           & $0.26$  & $\mathbf{0.58}$  & \nad & $+14.4\%$ \\
Cloudcast\,\dt intra (\$\,\dnar)        & GEPA & $\mathbf{102.1}$ & $157.1$ & $107.6$          & \nad & $-5.4\%$ tie \\
Cloudcast\,\dt inter (\$\,\dnar)        & GEPA & $213.3$          & $317.4$ & $\mathbf{201.8}$ & \nad & $+5.4\%$ \\
\addlinespace[3pt]
\famrow{Family 3\quad Approximate Bayesian inference (reparameterization)}\\
\addlinespace[1pt]
\blkrow{\textbf{F. Reparam suite, VI inner}\dt (1+1) LLM / VI (Adam/ELBO)\dt Gemini 3.5 Flash\dt metric: ELBO nats \upar\dt van $=$ untuned, hyb $=$ tuned\dt adv: nats}\\
\cmidrule(l{0.2em}r{0.2em}){1-7}
funnel \,($\Ccal^{+}$)         & Gemini 3.5 Flash & $0.08$   & \nad & $0.01$          & \nad & $-0.07$ tie \\
eight\_schools \,($\Ccal^{+}$) & Gemini 3.5 Flash & $2.14$   & \nad & $\mathbf{7.97}$ & \nad & $+5.8$ \\
gauss\_ill \,($\Ccal^{-}$)     & Gemini 3.5 Flash & $-0.34$  & \nad & $\mathbf{8.97}$ & \nad & $+9.3$ \\
gauss\_rot \,($\Ccal^{-}$)     & Gemini 3.5 Flash & $-33.4$  & \nad & $\mathbf{4.30}$ & \nad & $+37.7$ \\
\addlinespace[2pt]
\blkrow{\textbf{G. Reparam suite, NUTS inner}\dt (1+1) LLM / NUTS warmup\dt Gemini 3.5 Flash\dt metric: ESS/grad \upar\dt numerical: identity $+$ adapt.\dt adv: decades}\\
\cmidrule(l{0.2em}r{0.2em}){1-7}
funnel \,($\Ccal^{+}$)         & Gemini 3.5 Flash & \nad & \sci{8.4}{-5} & $\mathbf{0.111}$ & \nad & $+3.12$ \\
eight\_schools \,($\Ccal^{+}$) & Gemini 3.5 Flash & \nad & \sci{3.8}{-4} & $\mathbf{0.056}$ & \nad & $+2.17$ \\
gauss\_ill \,($\Ccal^{-}$)     & Gemini 3.5 Flash & \nad & $0.149$       & $0.151$          & \nad & $+0.01$ tie \\
gauss\_rot \,($\Ccal^{-}$)     & Gemini 3.5 Flash & \nad & $0.003$       & $0.003$          & \nad & $-0.02$ tie \\
\addlinespace[2pt]
\blkrow{\textbf{H. Horseshoe logistic regression} ($D{=}18$, hard case)\dt (1+1) LLM / NUTS warmup\dt metric: ESS/grad \upar\dt numerical: identity $+$ adapt.\,(\sci{6.2}{-4}, 520 div.)\dt adv: decades}\\
\cmidrule(l{0.2em}r{0.2em}){1-7}
horseshoe & Gemini 3.1 Pro          & \nad & \sci{6.2}{-4} & \bsci{1.48}{-2} & \nad & $+1.37$ \\
horseshoe & Claude Opus 4.8         & \nad & \sci{6.2}{-4} & \bsci{1.29}{-2} & \nad & $+1.32$ \\
horseshoe & Opus 4.8 (high reas.)   & \nad & \sci{6.2}{-4} & \bsci{7.4}{-3}  & \nad & $+1.07$ \\
horseshoe & GLM-5.2                 & \nad & \sci{6.2}{-4} & \bsci{6.4}{-3}  & \nad & $+1.01$ \\
horseshoe & Gemini 3.5 Flash        & \nad & \sci{6.2}{-4} & \bsci{6.2}{-3}  & \nad & $+1.00$ \\
\addlinespace[2pt]
\blkrow{\textbf{I. Curved-ridge ``banana''} ($D{=}10$, hard case)\dt (1+1) LLM / VI (then NUTS)\dt metric: ELBO nats \upar\dt numerical: identity ($-174$)\dt van $=$ guessed $c$, hyb $=$ tuned $c$\dt adv: nats; NUTS cross-solver $+1.57$ dec}\\
\cmidrule(l{0.2em}r{0.2em}){1-7}
banana & Gemini 3.1 Pro   & $-72.9$ & $-174$ & $\mathbf{-13.4}$ & \nad & $+59.5$ \\
banana & Gemini 3.5 Flash & $-74.1$ & $-174$ & $\mathbf{-13.2}$ & \nad & $+60.9$ \\
banana & Claude Opus 4.8  & $-74.1$ & $-174$ & $\mathbf{-13.4}$ & \nad & $+60.8$ \\
\end{longtable}
}

\end{document}